\documentclass[10pt,twocolumn,letterpaper]{article}

\usepackage[pagenumbers]{cvpr} 
\usepackage{cuted} 
\usepackage{capt-of}

\DeclareMathOperator{\tr}{tr}
\DeclareMathOperator{\diag}{diag}
\DeclareMathOperator{\SO}{SO}
\DeclareMathOperator{\so}{\mathfrak{so}}
\DeclareMathOperator{\SPD}{SPD}

\usepackage{amsthm}
\newtheorem{theorem}{Theorem}
\newtheorem{proposition}{Proposition}
\newtheorem{lemma}{Lemma}
\newtheorem{corollary}{Corollary}
\theoremstyle{definition}
\newtheorem{definition}{Definition}
\newtheorem{assumption}{Assumption}
\newtheorem{remark}{Remark}
\usepackage{appendix}

\usepackage{algorithm}
\usepackage{algorithmic}

\newcommand{\E}{\mathbb{E}}

\usepackage{amsmath}
\usepackage{amssymb}
\usepackage{bm}
\usepackage{graphicx}
\graphicspath{{figures/}}
 \usepackage{multirow}
\usepackage[table]{xcolor}
\usepackage{colortbl}

\definecolor{cvprblue}{rgb}{0.21,0.49,0.74}

\usepackage[pagebackref,breaklinks,colorlinks,allcolors=cvprblue]{hyperref}

\usepackage[capitalize]{cleveref}

\def\paperID{*****} 
\def\confName{CVPR}
\def\confYear{2026}

\title{K-GMRF: Kinetic Gauss-Markov Random Field\\for First-Principles Covariance Tracking on Lie Groups}

\author{
 ZhiMing Li \\
 School of Computer Science and Technology, Tianjin University \\
 {\tt\small 3022206093@tju.edu.cn}
}

\begin{document}

\maketitle
\begin{strip}
    \centering
    \includegraphics[width=\linewidth]{teaser.png}
    
    \captionof{figure}{
        \textbf{K-GMRF: Second-order tracking on Lie groups via symplectic integration.}
        \textbf{Left: Method Overview.} Observations generate \emph{whitened commutator torques} (derived from Thm~1), driving the covariance dynamics on the manifold. Unlike first-order updates, this preserves momentum to enable coasting.
        \textbf{Right: Hierarchical Validation.}
        \emph{Top (Synthetic):} Ellipse tracking under 30-frame occlusion. K-GMRF (red) maintains inertial coasting, whereas EMA (blue) freezes immediately.
        \emph{Middle (Real-world):} OTB sequences with severe motion blur (e.g., \texttt{BlurCar2}). Our method (red) accurately predicts trajectory curvature, eliminating the lag seen in EMA (blue).
        \emph{Bottom (Quantitative):} Validation of the zero-lag property (left) and steady-state error analysis (right).
        K-GMRF achieves \textbf{30$\times$ error reduction} on noiseless synthetic baselines and improves IoU from 0.55 to \textbf{0.74} on challenging non-stationary sequences.
    }
    \label{fig:teaser}
\end{strip}
\begin{abstract}
    Tracking non-stationary covariance matrices is fundamental to vision yet hindered by existing estimators that either neglect manifold constraints or rely on first-order updates, incurring inevitable phase lag during rapid evolution. \textbf{We propose K-GMRF, an online, training-free framework for covariance tracking} that reformulates the problem as forced rigid-body motion on Lie groups. Derived from the Euler--Poincar\'{e} equations, our method interprets observations as torques driving a latent angular velocity, propagated via a structure-preserving symplectic integrator. We theoretically prove that this second-order dynamics achieves zero steady-state error under constant rotation, strictly superior to the proportional lag of first-order baselines. Validation across three domains demonstrates robust tracking fidelity: (i) on synthetic ellipses, K-GMRF reduces angular error by 30$\times$ ($15.6^\circ \to 0.51^\circ$) compared to Riemannian EMA while maintaining stability at high speeds; (ii) on SO(3) stabilization with 20\% dropout, it decreases geodesic error from $29.4^\circ$ to $9.9^\circ$; and (iii) on OTB motion-blur sequences, it improves IoU from 0.55 to 0.74 on BlurCar2 with a 96\% success rate. As a fully differentiable symplectic module, K-GMRF provides a plug-and-play geometric prior for data-constrained scenarios and an interpretable layer within modern deep architectures.
    \end{abstract}    

\section{Introduction}
\label{sec:intro}

Covariance matrices encode second-order statistics fundamental to computer vision: region covariance descriptors~\cite{tuzel2006region} capture joint feature distributions for detection and tracking; diffusion tensors model anisotropic tissue structure in medical imaging~\cite{pennec2006riemannian}; and sample covariances underpin texture classification, action recognition, and visual attention~\cite{huang2018geometry}. A recurring challenge is \emph{tracking non-stationary covariances}---estimating how the underlying positive-definite matrix evolves over time from noisy observations.

The symmetric positive-definite (SPD) constraint precludes naive Euclidean operations. Linearly averaging SPD matrices produces the ``swelling effect''~\cite{pennec2006riemannian}, motivating Riemannian approaches~\cite{gao2020learning, huang2017riemannian, brooks2019riemannian}. Yet most methods rely on first-order updates (EMA or gradient descent) that suffer inherent phase lag when the target rotates at constant angular velocity---a structural limitation of overdamped dynamics.

Kalman filtering on Lie groups~\cite{bourmaud2013discrete, sjoberg2020lie, brossard2017unscented} introduces second-order dynamics via a velocity state, enabling prediction during dropout. However, these approaches linearize the manifold and apply standard Kalman updates, sacrificing intrinsic geometry. The connection to the underlying likelihood remains implicit.

We propose the \emph{Kinetic Gauss-Markov Random Field} (K-GMRF), reformulating covariance tracking as forced rigid-body motion on the isospectral manifold $\mathcal{O}_\Lambda$. We derive the update equations from the Euler--Poincar\'{e} formalism~\cite{marsden1999introduction, holm1998euler}, proving that observations act as natural-gradient torques~\cite{amari1998natural} on the Lie algebra. The resulting symplectic integrator~\cite{hairer2006geometric} preserves manifold structure while enabling zero-lag tracking and inertial coasting through occlusions.

\vspace{0.3em}
\noindent\textbf{Contributions.}
\begin{itemize}[leftmargin=*, itemsep=2pt, topsep=2pt]
    \item \textbf{Mechanics-based formulation.} We derive K-GMRF from Euler--Poincar\'{e} equations, proving the \emph{whitened commutator torque} $\tau = S^{-1}[C,M]S^{-1}$ equals the natural gradient of the Wishart likelihood on $\so(d)$ (Theorem~\ref{thm:torque-gradient}).
    
    \item \textbf{Zero-lag with quantified bounds.} We prove K-GMRF achieves zero steady-state error in the stability domain $\mathcal{D} = \{(\eta,\gamma): 0 < \gamma < 2,\, \eta < 2(2-\gamma)/\kappa_{\max}\}$ (Theorem~\ref{thm:zero-lag}), while any first-order method incurs lag $\geq |\Omega^*|/4$ (Theorem~\ref{thm:ema-lag}). The Lyapunov energy contracts as $\E[\mathcal{E}_{t+1}] \leq (1-\kappa)\mathcal{E}_t + \sigma_{\mathrm{eff}}^2$ with $\sigma_{\mathrm{eff}}^2 = O(\sigma^2/(m\Delta_{\mathrm{wh}}^2))$.
    
    \item \textbf{Rate-optimal tracking bound.} The time-averaged risk satisfies $\mathcal{R}_T \leq C_1\sigma^2/(m\Delta_{\mathrm{wh}}^2) + C_2 V_\Omega/(\gamma T) + C_3 d_0^2(1-\kappa)^T/T$ (Theorem~\ref{thm:master}), matching the minimax lower bound in leading terms (Theorem~\ref{thm:minimax}). The phase transition occurs at $\Delta_{\mathrm{wh}} \to 0$.
    
    \item \textbf{Empirical validation.} We validate on synthetic and real benchmarks: $30\times$ error reduction on SPD(2) ellipse tracking, $3\times$ improvement on SO(3) stabilization with 20\% dropout, and 0.74 IoU (vs.\ 0.55) on OTB motion-blur sequences.
\end{itemize}

\vspace{0.3em}
\noindent\textbf{Positioning.} Rather than competing with deep trackers in rich-data regimes, K-GMRF targets \emph{data-constrained} and \emph{interpretability-critical} applications---medical imaging, scientific discovery, or online adaptation without pre-training. Its derivation from first principles yields a \emph{differentiable symplectic layer} that injects physical consistency into deep networks: the Kick-Drift-Measure integrator can be unrolled as a recurrent module, enabling end-to-end training while preserving manifold structure by construction.

\section{Related Work}
\label{sec:related}

\noindent\textbf{Covariance descriptors in visual tracking.}
The seminal work of Tuzel~\etal~\cite{tuzel2006region} introduced region covariance descriptors for object detection and tracking, encoding spatial and photometric statistics in a compact SPD matrix. Subsequent extensions integrated covariance descriptors with particle filtering~\cite{romero2012covariance} and clustering-based model update~\cite{qin2014object}. Porikli~\etal~\cite{porikli2006covariance} pioneered the use of Lie algebra for incremental covariance updates, establishing an early connection between Riemannian geometry and tracking. However, these methods rely on first-order averaging that ignores the curvature of the SPD manifold and lacks predictive capability during occlusions.

\vspace{0.3em}
\noindent\textbf{Riemannian geometry on SPD manifolds.}
Pennec~\etal~\cite{pennec2006riemannian} established the foundational Riemannian framework for tensor computing, demonstrating the ``swelling effect'' of Euclidean averaging on SPD matrices. This motivated a rich body of work on Riemannian operations for SPD matrices, including Fr\'echet means~\cite{pennec2006riemannian}, Log-Euclidean mappings~\cite{pennec2006riemannian}, and affine-invariant metrics~\cite{huang2018geometry}. Recent advances have embedded Riemannian geometry into deep networks: SPDNet~\cite{huang2017riemannian} introduced bilinear mappings and eigenvalue rectification; Brooks~\etal~\cite{brooks2019riemannian} developed Riemannian batch normalization; and U-SPDNet~\cite{wang2023uspd} addressed information degradation through skip connections. Chen~\etal~\cite{chen2023riemannian} further proposed local geometric mechanisms for SPD networks. These methods achieve impressive results in classification but do not address the temporal dynamics required for online tracking of non-stationary covariances.

\vspace{0.3em}
\noindent\textbf{Filtering on Lie groups.}
State estimation on Lie groups has received significant attention in robotics. Sol{\`a}~\etal~\cite{sola2018micro} provided a comprehensive tutorial on Lie theory for state estimation, while Bourmaud~\etal~\cite{bourmaud2013discrete} developed discrete extended Kalman filters on Lie groups. Brossard~\etal~\cite{brossard2017unscented} extended the unscented Kalman filter to Lie groups with applications to visual-inertial odometry. The invariant extended Kalman filter (IEKF)~\cite{barrau2017invariant} exploits symmetry to achieve improved convergence guarantees. Sj{\o}berg and Egeland~\cite{sjoberg2020lie} proposed Lie algebraic UKF for pose estimation. These methods maintain velocity/acceleration states, enabling prediction during sensor dropout. However, they typically linearize the manifold into tangent space and apply standard Kalman updates, sacrificing the intrinsic Riemannian structure. Furthermore, the innovation term lacks a principled connection to the likelihood function on the manifold.

\vspace{0.3em}
\noindent\textbf{Hamiltonian mechanics and symplectic integration.}
Geometric numerical integration~\cite{hairer2006geometric} has demonstrated the superiority of structure-preserving algorithms for long-horizon simulation of Hamiltonian systems. Marsden and West~\cite{marsden1999discrete} established variational integrators that exactly preserve symplectic structure. In machine learning, Greydanus~\etal~\cite{greydanus2019hamiltonian} introduced Hamiltonian neural networks that learn the Hamiltonian from data; SympNets~\cite{jin2020sympnets} architecturally encode symplecticity for improved generalization. For rigid body control, the Euler--Poincar\'e reduction~\cite{holm1998euler, marsden1999introduction} provides a principled framework to derive equations of motion on Lie groups. Wi\'sniewski and Kulczycki~\cite{wisniewski2004euler} applied this framework to externally forced rigid body motion. Our work bridges this mechanical formalism to statistical estimation, deriving the observation-induced torque from the Wishart likelihood's natural gradient.

\vspace{0.3em}
\noindent\textbf{Natural gradient and information geometry.}
Amari~\cite{amari1998natural} introduced the natural gradient, which preconditions gradient descent with the Fisher information matrix to achieve reparametrization invariance. Martens~\cite{martens2020new} provided modern perspectives connecting natural gradient to second-order optimization. In our framework, the ``whitened commutator torque'' arising from the Euler--Poincar\'e equations is shown to be precisely the natural gradient of the Wishart log-likelihood projected onto the Lie algebra (Theorem~\ref{thm:torque-gradient}). This establishes a principled bridge between information geometry and Hamiltonian mechanics for covariance tracking.

\vspace{0.3em}
\noindent\textbf{EMA-based model update in tracking.}
Exponential moving average (EMA) remains the dominant paradigm for template update in visual tracking. Huang and Zhou~\cite{huang2019re2ema} analyzed EMA from an optimization perspective, proposing regularization and reinitialization to mitigate drift. However, EMA is fundamentally a first-order method: it cannot anticipate target motion and exhibits inherent phase lag when tracking rotating targets (Theorem~\ref{thm:ema-lag}). Our K-GMRF addresses this by introducing momentum on the Lie algebra, enabling zero-lag tracking (Theorem~\ref{thm:zero-lag}) and inertial coasting during observation dropout.

\section{Kinetic Gauss-Markov Random Field}
\label{sec:method}

\subsection{Geometric Setup and Observation Model}
\label{sec:setup}

\begin{definition}[Isospectral Orbit]
\label{def:orbit}
Given $\Lambda = \diag(\lambda_1, \ldots, \lambda_d) \in \mathbb{S}_{++}^d$ with $\lambda_1 > \cdots > \lambda_d > 0$, the isospectral orbit is
\begin{equation}
    \mathcal{O}_\Lambda := \{Q\Lambda Q^\top : Q \in \SO(d)\} \subset \mathbb{S}_{++}^d.
\end{equation}
The tangent space at $M \in \mathcal{O}_\Lambda$ is $T_M\mathcal{O}_\Lambda = \{[\Omega, M] : \Omega \in \so(d)\}$, where $[\cdot, \cdot]$ denotes the matrix commutator.
\end{definition}

\begin{definition}[Whitened Spectral Gap]
\label{def:spectral-gap}
The identifiability parameter for rotation recovery under observation noise $\sigma^2$ is
\begin{equation}
    \Delta_{\mathrm{wh}} := \min_{i \neq j} \left| \frac{\lambda_i}{\lambda_i + \sigma^2} - \frac{\lambda_j}{\lambda_j + \sigma^2} \right|.
    \label{eq:spectral-gap}
\end{equation}
\end{definition}

\begin{definition}[Whitened Commutator Torque]
\label{def:torque}
For state $M \in \mathcal{O}_\Lambda$, observation $C \in \mathbb{S}_+^d$, and whitened covariance $S := M + \sigma^2 I$:
\begin{equation}
    \tau(M, C) := S^{-1}[C, M]S^{-1} \in \so(d).
    \label{eq:torque}
\end{equation}
\end{definition}

\begin{assumption}[Wishart Observation Model~\cite{ayadi2024elliptical}]
\label{ass:wishart}
The observation $C_t$ is generated by $m C_t \sim \mathcal{W}_d(m, M_t^* + \sigma^2 I)$, where $M_t^*$ is the latent state. Thus $\E[C_t \mid M_t^*] = M_t^* + \sigma^2 I$.
\end{assumption}

\subsection{Symplectic Kick-Drift-Measure Integrator}
\label{sec:algorithm}

The K-GMRF update~\cite{hairer2006geometric, marsden1999discrete} maintains state $(M_t, \Omega_t) \in \mathcal{O}_\Lambda \times \so(d)$. The complete algorithm is given in Algorithm~\ref{alg:kgmrf}.

\begin{algorithm}[t]
\caption{K-GMRF: Kinetic Gauss-Markov Random Field Tracker}
\label{alg:kgmrf}
\begin{algorithmic}[1]
\REQUIRE Initial state $M_0 \in \mathcal{O}_\Lambda$,  $\Omega_0 = 0 \in \so(d)$
\REQUIRE Noise scale $\sigma^2 > 0$,  $\eta > 0$, damping $\gamma \in (0,1)$
\ENSURE Sequence of estimates $\{M_t\}_{t=1}^T$
\FOR{$t = 0, 1, \ldots, T-1$}
    \STATE $S_t \gets M_t + \sigma^2 I$ \hfill $\triangleright$ Whitened covariance
    \IF{observation $C_t$ is available}
        \STATE $\tau_t \gets S_t^{-1}(C_t M_t - M_t C_t)S_t^{-1}$ \hfill $\triangleright$ \textbf{Measure}: Compute torque
    \ELSE
        \STATE $\tau_t \gets 0$ \hfill $\triangleright$ \textbf{Coast}: No observation (occlusion)
    \ENDIF
    \STATE $\Delta\Omega \gets \mathcal{I}_{M_t}^{-1}(\tau_t)$ \hfill $\triangleright$ Apply inertia inverse (Eq.~\ref{eq:inertia})
    \STATE $\Omega_{t+1} \gets (1 - \gamma)\Omega_t + \eta \cdot \Delta\Omega$ \hfill $\triangleright$ \textbf{Kick}: Update velocity
    \STATE $R \gets \exp(\Omega_{t+1})$ \hfill $\triangleright$ Matrix exponential
    \STATE $M_{t+1} \gets R \, M_t \, R^\top$ \hfill $\triangleright$ \textbf{Drift}: Rotate on manifold
\ENDFOR
\RETURN $\{M_t\}_{t=1}^T$
\end{algorithmic}
\end{algorithm}

\begin{remark}[Why Gauss-Markov Random Field?]
\label{rem:gmrf}
In Lie algebra coordinates $(\xi_t, \omega_t) \in \so(d) \times \so(d)$, where $\xi_t = \log_{M^*}(M_t)$ is the configuration and $\omega_t$ is angular velocity, the constant-velocity prior
\begin{equation}
\textstyle
E_{\mathrm{prior}} = \sum_t \|\xi_t - \omega_t\|^2 + \|\omega_{t+1} - \omega_t\|^2
\end{equation}
induces a \emph{block-tridiagonal precision matrix}---the hallmark of a Gauss-Markov Random Field~\cite{rue2005gaussian}. K-GMRF thus generalizes classical first-order GMRF (random walk prior on position) to a second-order (constant-velocity prior) 
\end{remark}

\textbf{Inertia inverse operator.} Given eigendecomposition $M = U \Lambda U^\top$, define $d_i := \lambda_i + \sigma^2$. For torque $\tau \in \so(d)$, let $\tilde{\tau} := U^\top \tau U$. The inertia inverse is computed element-wise:
\begin{equation}
    \left[\mathcal{I}_M^{-1}(\tau)\right]_{ij} = \frac{d_i d_j}{(\lambda_i - \lambda_j)^2 + \epsilon} \cdot \tilde{\tau}_{ij}, \quad i \neq j,
    \label{eq:inertia}
\end{equation}
where $\epsilon > 0$ is a regularization term preventing singularity when eigenvalues coincide.

\textbf{Computational complexity.} Each iteration requires $O(d^3)$ for the eigen decomposition yet Cayley--Neumann integration reduces this to $O(Kd^2)$ cost per frame

\subsection{Torque as Natural Gradient Projection}
\label{sec:torque-gradient}

\begin{theorem}[Geometric Consistency~\cite{amari1998natural, pennec2006riemannian}]
\label{thm:torque-gradient}
Under the affine-invariant Riemannian metric (AIRM) on $\mathcal{O}_\Lambda$, the torque \eqref{eq:torque} equals the natural gradient of the Wishart negative log-likelihood projected onto $\so(d)$:
\begin{equation}
    \mathrm{grad}_{\mathcal{O}} V(M; C) = \left[\mathcal{I}_M^{-1}\left(\tfrac{m}{2}\tau(M,C)\right), M\right],
\end{equation}
where $V(M;C) = \frac{m}{2}(\log\det S + \tr(S^{-1}C))$ and $\mathcal{I}_M: \so(d) \to \so(d)$ is the inertia operator defined by $\langle \mathcal{I}_M(\Omega_1), \Omega_2 \rangle_F = g_M^{\mathcal{O}}([\Omega_1, M], [\Omega_2, M])$.
\end{theorem}
\textit{Proof.} See Supplementary Material A.3. \hfill $\square$

\begin{lemma}[Unbiasedness and Concentration]
\label{lem:unbiased}
Under Assumption~\ref{ass:wishart}, the torque satisfies:\\
(i) \emph{Conditional unbiasedness}: $\E[\tau(M, C_t) \mid M_t^*] = \tau(M, M_t^* + \sigma^2 I)$, with $\tau(M^*, S^*) = 0$.\\
(ii) \emph{Variance bound}: $\E[\|\tau - \E[\tau]\|_F^2] \leq K_\tau(\Lambda, \sigma^2) / m$, where $K_\tau$ depends only on the spectral structure.
\end{lemma}
\textit{Proof.} See Supplementary Material A.4. \hfill $\square$

\begin{lemma}[Local Strong Convexity~\cite{pennec2006riemannian}]
\label{lem:identifiability}
On orbit neighborhoods, the Hessian of expected negative log-likelihood satisfies
\begin{equation}
    \mathrm{Hess}_{\mathcal{O}} \bar{V}(M^*) \succeq \mu(\Delta_{\mathrm{wh}}) \cdot I, \quad \text{where } \mu(\Delta_{\mathrm{wh}}) > 0 \iff \Delta_{\mathrm{wh}} > 0.
\end{equation}
The curvature $\mu(\Delta_{\mathrm{wh}}) \to 0$ as $\Delta_{\mathrm{wh}} \to 0$, characterizing the phase transition boundary.
\end{lemma}
\textit{Proof.} See Supplementary Material A.5. \hfill $\square$

\begin{proposition}[Structure Preservation~\cite{marsden1999discrete, hairer2006geometric}]
\label{prop:structure}
The undamped K-GMRF ($\gamma = 0$) is a first-order Lie group variational integrator:\\
(i) \emph{Orbit invariance}: $M_0 \in \mathcal{O}_\Lambda \Rightarrow M_t \in \mathcal{O}_\Lambda$ for all $t \geq 0$.\\
(ii) \emph{Symplecticity}: The update preserves a discrete symplectic form consistent with Euler--Poincar\'e dynamics, with local truncation error $O(\eta^2)$.
\end{proposition}
\textit{Proof.} See Supplementary Material A.6. \hfill $\square$

\subsection{Zero Steady-State Error under Constant Rotation}
\label{sec:zerolag}

The central theoretical contribution establishes \emph{zero steady-state error} for the second-order K-GMRF dynamics, in contrast to the \emph{inevitable lag} of first-order methods.

\begin{theorem}[Zero-Lag Tracking]
\label{thm:zero-lag}
Consider noiseless observations $C_t = M_t^*$ and constant angular velocity $\Omega_t^* \equiv \Omega^*$. Define the stability domain
\begin{equation}
    \mathcal{D} := \left\{(\eta, \gamma) : 0 < \gamma < 2, \quad 0 < \eta < \frac{2(2-\gamma)}{\kappa_{\max}}\right\},
\end{equation}
where $\kappa_{\max} = \max \mathrm{eig}(\mathcal{I}^{-1}\mathcal{K})$ and $\mathcal{K}$ is the stiffness operator from Lemma~\ref{lem:identifiability}. Then for $(\eta, \gamma) \in \mathcal{D}$:
\begin{equation}
    \lim_{t \to \infty} d_{\mathcal{O}}(M_t, M_t^*) = 0 \quad \textit{(zero steady-state error)}.
\end{equation}
Convergence is locally exponential with rate determined by the spectral radius of the linearized error dynamics.
\end{theorem}
\textit{Proof.} The proof proceeds by linearizing the error dynamics around the equilibrium trajectory $(M_t^*, \Omega^*)$. The characteristic polynomial of the resulting discrete damped harmonic oscillator yields the Schur stability conditions defining $\mathcal{D}$. See Supplementary Material B.1 for complete details. \hfill $\square$

\begin{theorem}[Inevitable Lag of First-Order Methods]
\label{thm:ema-lag}
Consider any first-order method (exponential moving average, Riemannian gradient descent) without explicit velocity state, evolving as $M_{t+1} = \exp(-\eta \mathcal{A}^{-1} \tau_t) M_t \exp(\cdot)^\top$. Under the same constant-rotation setting as Theorem~\ref{thm:zero-lag}, the steady-state error satisfies
\begin{equation}
    \liminf_{t \to \infty} d_{\mathcal{O}}(M_t, M_t^*) \geq \frac{|\Omega^*|_F}{2\eta \bar{\kappa}},
\end{equation}
where $\bar{\kappa} = \max \mathrm{eig}(\mathcal{A})$ and stability requires $\eta \bar{\kappa} < 2$. Combining these constraints:
\begin{equation}
    \liminf_{t \to \infty} d_{\mathcal{O}}(M_t, M_t^*) \geq \frac{|\Omega^*|_F}{4}.
\end{equation}
Zero steady-state error is structurally unachievable for any stable first-order method.
\end{theorem}
\textit{Proof.} The first-order dynamics require nonzero position error to generate corrective torque. Balancing the driving term against stability constraints yields the stated lower bound. See Supplementary Material B.2. \hfill $\square$

\subsection{Stochastic Stability and Tracking Error Bound}
\label{sec:master}

\begin{theorem}[Energy Contraction to Noise Ball~\cite{deng2001stabilization}]
\label{thm:stability}
Define the Lyapunov energy $\mathcal{E}_t := T(u_t) + \frac{1}{2}|\xi_t|_F^2$, where $u_t := \Omega_t - \Omega^*$ is the velocity error and $\xi_t := \log_{M^*}(M_t)$ is the configuration error. Under $(\eta, \gamma) \in \mathcal{D}$ and Assumption~\ref{ass:wishart}, there exist constants $\kappa \in (0,1)$ and
\begin{equation}
    \sigma_{\mathrm{eff}}^2 := \frac{C_0 \sigma^2}{m \Delta_{\mathrm{wh}}^2}
\end{equation}
such that the conditional expectation satisfies
\begin{equation}
    \E[\mathcal{E}_{t+1} \mid \mathcal{F}_t] \leq (1-\kappa)\mathcal{E}_t + \sigma_{\mathrm{eff}}^2.
    \label{eq:energy-contraction}
\end{equation}
Iterating yields $\sup_t \E[\mathcal{E}_t] \leq \sigma_{\mathrm{eff}}^2 / \kappa$ (convergence to noise ball).
\end{theorem}
\textit{Proof.} The proof combines Lemmas~\ref{lem:unbiased}--\ref{lem:identifiability} with a discrete Lyapunov argument. The contraction rate $\kappa$ depends on $(\eta, \gamma, \Delta_{\mathrm{wh}})$. See Supplementary Material B.3. \hfill $\square$

\begin{theorem}[Master Theorem: Time-Averaged Tracking Risk]
\label{thm:master}
Let $V_\Omega := \sum_{t=1}^{T-1} |\Omega_{t+1}^* - \Omega_t^*|_F$ denote the total variation of target angular velocity. Under Assumption~\ref{ass:wishart} and $(\eta, \gamma) \in \mathcal{D}$, the time-averaged risk
\begin{equation}
    \mathcal{R}_T := \frac{1}{T}\sum_{t=1}^T \E\big[d_{\mathcal{O}}(M_t, M_t^*)^2\big]
\end{equation}
satisfies the following bound:
\begin{equation}
\boxed{
\begin{aligned}
    \mathcal{R}_T \;\leq\; & \underbrace{\frac{C_1 \sigma^2}{m \Delta_{\mathrm{wh}}^2}}_{\text{statistical}} \;+\; \underbrace{\frac{C_2 V_\Omega}{\gamma T}}_{\text{nonstationarity}} \;+\; \underbrace{\frac{C_3 d_0^2 (1-\kappa)^T}{T}}_{\text{transient}}
\end{aligned}
}
\label{eq:master-bound}
\end{equation}
where $C_1, C_2, C_3$ depend on $(\Lambda, \sigma^2, \eta, \gamma)$ but not on $(m, T, V_\Omega)$.

\textbf{Phase transition.} When $\Delta_{\mathrm{wh}} \to 0$, the statistical error diverges. When $(\eta, \gamma) \notin \mathcal{D}$, the transient term fails to decay.
\end{theorem}
\textit{Proof.} The bound combines Theorem~\ref{thm:stability} with a perturbation analysis for time-varying $\Omega_t^*$. The nonstationarity term arises from Lemma C (target variation perturbation) in the Supplementary Material. See Supplementary Material B.4. \hfill $\square$

\subsection{Generalization to Noncommutative Rotations and Optimality}
\label{sec:extensions}

\begin{corollary}[Extension to $d \geq 3$: Noncommutative Lie Algebra]
\label{cor:noncommutative}
For $d \geq 3$, the Lie algebra $\so(d)$ is noncommutative, introducing Baker--Campbell--Hausdorff (BCH)~\cite{hall2015lie} remainder terms in the error dynamics. Theorems~\ref{thm:zero-lag}--\ref{thm:master} remain valid with the following modifications:
\begin{enumerate}[topsep=2pt, itemsep=1pt]
    \item The energy contraction \eqref{eq:energy-contraction} acquires an additional term:
    \begin{equation}
        \E[\mathcal{E}_{t+1} \mid \mathcal{F}_t] \leq (1-\kappa')\mathcal{E}_t + \sigma_{\mathrm{eff}}^2 + C_{\mathrm{BCH}} |\Omega^*|^2 \mathcal{E}_t,
    \end{equation}
    where $C_{\mathrm{BCH}} = O(1)$ is a universal constant from the BCH expansion.
    \item The stability domain $\mathcal{D}$ contracts by a factor depending on $|\Omega^*|$ to absorb the BCH perturbation.
    \item The master bound \eqref{eq:master-bound} holds with modified constants $C_1' = C_1(1 + O(d|\Omega^*|^2))$.
\end{enumerate}
Critically, the phase transition is still governed by $\Delta_{\mathrm{wh}}$; noncommutativity affects only the constants.
\end{corollary}
\textit{Proof.} The proof uses Lemma 9.1 (BCH remainder bound $O(r^3)$), Lemma 9.2 (adjoint Lipschitz), and Lemma 9.3 (multi-plane rotation decomposition) to control higher-order terms. See Supplementary Material C.1. \hfill $\square$

\begin{table*}[t]
\centering
\caption{\textbf{Second-order dynamics + manifold structure = robustness.} We compare methods along two axes: filter order (1st vs.\ 2nd) and geometry (Euclidean vs.\ Riemannian). The ``Gain'' column reports improvement over Riemannian EMA.}
\label{tab:main-results}
\vspace{0.3em}
\small
\setlength{\tabcolsep}{3.5pt}
\renewcommand{\arraystretch}{1.15}
\begin{tabular}{l l ccccc c}
\toprule
& \textbf{Setting} & \textbf{K-GMRF (Ours)} & \textbf{Riem.\ EMA} & \textbf{Eucl.\ EMA} & \textbf{Tangent KF} & \textbf{Alpha-Beta} & \textbf{Gain} \\
\midrule
\rowcolor{blue!8}
\multicolumn{8}{l}{\textbf{(A) Rotating Ellipse on SPD(2):} Synthetic covariance tracking, $\omega{=}0.08$ rad/step, 400 frames. \textit{Angular Error (deg)} $\downarrow$} \\
\rowcolor{blue!4}
& Normal           & $\mathbf{0.51}{\scriptstyle\pm0.03}$ & $15.62{\scriptstyle\pm0.06}$ & $15.62{\scriptstyle\pm0.06}$ & $0.84{\scriptstyle\pm0.02}$ & $2.41{\scriptstyle\pm0.06}$ & \textbf{30$\times$} \\
\rowcolor{blue!4}
& 20\% Dropout     & $\mathbf{12.07}{\scriptstyle\pm0.26}$ & $25.29{\scriptstyle\pm0.04}$ & $25.29{\scriptstyle\pm0.04}$ & $11.97{\scriptstyle\pm0.34}$ & $16.58{\scriptstyle\pm0.06}$ & 2.1$\times$ \\
\midrule
\rowcolor{green!8}
\multicolumn{8}{l}{\textbf{(B) Camera Stabilization on SO(3):} Coupled oscillations with observation dropout, 200 frames. \textit{Geodesic Error (deg)} $\downarrow$} \\
\rowcolor{green!4}
& 0\% Dropout      & $\mathbf{4.4}{\scriptstyle\pm0.2}$ & $7.2{\scriptstyle\pm0.2}$ & $7.3{\scriptstyle\pm0.2}$ & $8.1{\scriptstyle\pm0.6}$ & $4.4{\scriptstyle\pm0.2}$ & 1.7$\times$ \\
\rowcolor{green!4}
& 10\% Dropout     & $\mathbf{5.8}{\scriptstyle\pm0.4}$ & $18.7{\scriptstyle\pm2.3}$ & $19.0{\scriptstyle\pm2.5}$ & $16.2{\scriptstyle\pm3.4}$ & $5.8{\scriptstyle\pm0.3}$ & 3.2$\times$ \\
\rowcolor{green!4}
& 20\% Dropout     & $\mathbf{6.5}{\scriptstyle\pm0.6}$ & $29.2{\scriptstyle\pm4.1}$ & $30.2{\scriptstyle\pm4.5}$ & $22.5{\scriptstyle\pm4.7}$ & $6.6{\scriptstyle\pm1.2}$ & \textbf{4.5$\times$} \\
\rowcolor{green!4}
& 30\% Dropout     & $\mathbf{8.0}{\scriptstyle\pm1.1}$ & $41.0{\scriptstyle\pm2.1}$ & $43.5{\scriptstyle\pm2.4}$ & $32.1{\scriptstyle\pm7.2}$ & $8.1{\scriptstyle\pm1.6}$ & 5.1$\times$ \\
\rowcolor{green!4}
& 40\% Dropout     & $\mathbf{14.3}{\scriptstyle\pm3.2}$ & $56.5{\scriptstyle\pm10.3}$ & $63.3{\scriptstyle\pm13.3}$ & $38.0{\scriptstyle\pm8.6}$ & $14.9{\scriptstyle\pm5.1}$ & 4.0$\times$ \\
\midrule
\rowcolor{orange!8}
\multicolumn{8}{l}{\textbf{(C) OTB Motion-Blur Sequences:} Real video tracking with $7{\times}7$ region covariance descriptors. \textit{IoU} $\uparrow$} \\
\rowcolor{orange!4}
& BlurBody         & $\mathbf{0.65}$ & $0.64$ & $0.62$ & $0.35$ & $0.40$ & +2\% \\
\rowcolor{orange!4}
& BlurCar1         & $\mathbf{0.42}$ & $0.40$ & $0.41$ & $0.26$ & $0.40$ & +5\% \\
\rowcolor{orange!4}
& BlurCar2         & $\mathbf{0.74}$ & $0.55$ & $0.66$ & $0.45$ & $0.50$ & \textbf{+35\%} \\
\rowcolor{orange!4}
& BlurFace         & $0.84$ & $\mathbf{0.86}$ & $0.84$ & $0.63$ & $0.68$ & -- \\
\rowcolor{orange!4}
& CarScale         & $0.65$ & $0.56$ & $\mathbf{0.66}$ & $0.09$ & $0.55$ & +16\% \\
\rowcolor{orange!4}
& Jogging          & $0.48$ & $0.49$ & $\mathbf{0.50}$ & $0.45$ & $0.33$ & -- \\
\bottomrule
\end{tabular}
\vspace{-0.5em}
\end{table*}
\begin{theorem}[Minimax Lower Bound and Rate Optimality~\cite{tsybakov2009introduction}]
\label{thm:minimax}
Consider the class of fixed-gain online filters $\{\phi_t\}$ over the problem family $\mathcal{P}(\Lambda, \sigma^2, m, V_\Omega)$. The minimax risk satisfies
\begin{equation}
    \inf_{\{\phi_t\}} \sup_{\mathcal{P}} \mathcal{R}_T \geq c_1 \frac{\sigma^2}{m \Delta_{\mathrm{wh}}^2} + c_2 \frac{V_\Omega}{T}
\end{equation}
for universal constants $c_1, c_2 > 0$. Comparing with Theorem~\ref{thm:master}, K-GMRF achieves the minimax rate in both the statistical ($1/m$) and nonstationarity ($V_\Omega/T$) terms.
\end{theorem}
\textit{Proof.} The lower bound follows from a two-point testing argument combined with Fano's inequality applied to the orbit distance. The $1/m$ term arises from Fisher information of the Wishart model; the $V_\Omega/T$ term from a random walk adversary. See Supplementary Material C.2. \hfill $\square$

\section{Experiments}
\label{sec:experiments}

We validate K-GMRF on three tasks: (A)~synthetic ellipse tracking on $\SPD(2)$, (B)~camera stabilization on $\SO(3)$, and (C)~real-world OTB tracking. All experiments use seed separation: hyperparameters tuned on 5 seeds, evaluated on 5 independent seeds.

\subsection{Setup}

\noindent\textbf{Baselines.} We compare four methods spanning the design space of filter order and geometry:
\emph{Riemannian EMA} (1st-order, manifold),
\emph{Euclidean EMA} (1st-order, Euclidean),
\emph{Tangent KF} (2nd-order, tangent space linearization),
\emph{Alpha-Beta} ~\cite{benedict1962synthesis} (2nd-order, Euclidean).
This design isolates the effect of momentum (1st vs.\ 2nd order) from the effect of geometry (Euclidean vs.\ Riemannian).

\noindent\textbf{Data.}
(A)~Rotating ellipse: covariance $M^*(t) = Q(t)\Lambda Q(t)^\top$ with Wishart-distributed observations.
(B)~Camera stabilization: rotation matrices on $\SO(3)$ with coupled oscillations and random dropout.
(C)~OTB sequences~\cite{wu2013online}: 6 motion-blur videos (BlurBody, BlurCar1/2, BlurFace, CarScale, Jogging) using $7{\times}7$ region covariance descriptors~\cite{tuzel2006region}.

\noindent\textbf{Metrics.} Angular error (degrees) for synthetic; IoU for OTB.

\subsection{Results and Analysis}

Table~\ref{tab:main-results} summarizes all results. Three patterns emerge:

\noindent\textbf{(1) Second-order $\gg$ first-order.}
On SPD(2), K-GMRF achieves $0.51^\circ$ error versus $15.62^\circ$ for Riemannian EMA (30$\times$ reduction), confirming Theorem~\ref{thm:zero-lag}.
On SO(3) with 20\% dropout, both K-GMRF ($6.5^\circ$) and Alpha-Beta ($6.6^\circ$) outperform first-order methods by 4--5$\times$.
The gap widens under higher dropout: at 40\%, second-order methods remain stable while EMA variants degrade rapidly.

\begin{figure}[t]
\centering
\includegraphics[width=0.95\linewidth]{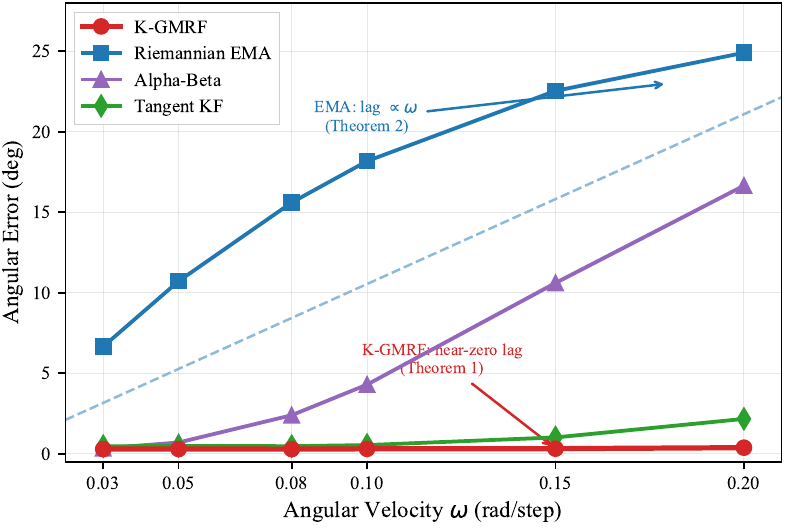}
\vspace{-0.5em}

\caption{\textbf{Angular velocity sweep} validating Theorems~\ref{thm:zero-lag} and~\ref{thm:ema-lag}. As target speed $\omega$ increases, EMA error grows linearly (lag $\propto \omega$), while K-GMRF maintains $<0.4^\circ$ error across all speeds. Alpha-Beta also degrades at high $\omega$.}
\label{fig:omega-sweep}
\vspace{-0.5em}
\end{figure}

\begin{figure*}[t]
\centering
\includegraphics[width=0.98\textwidth]{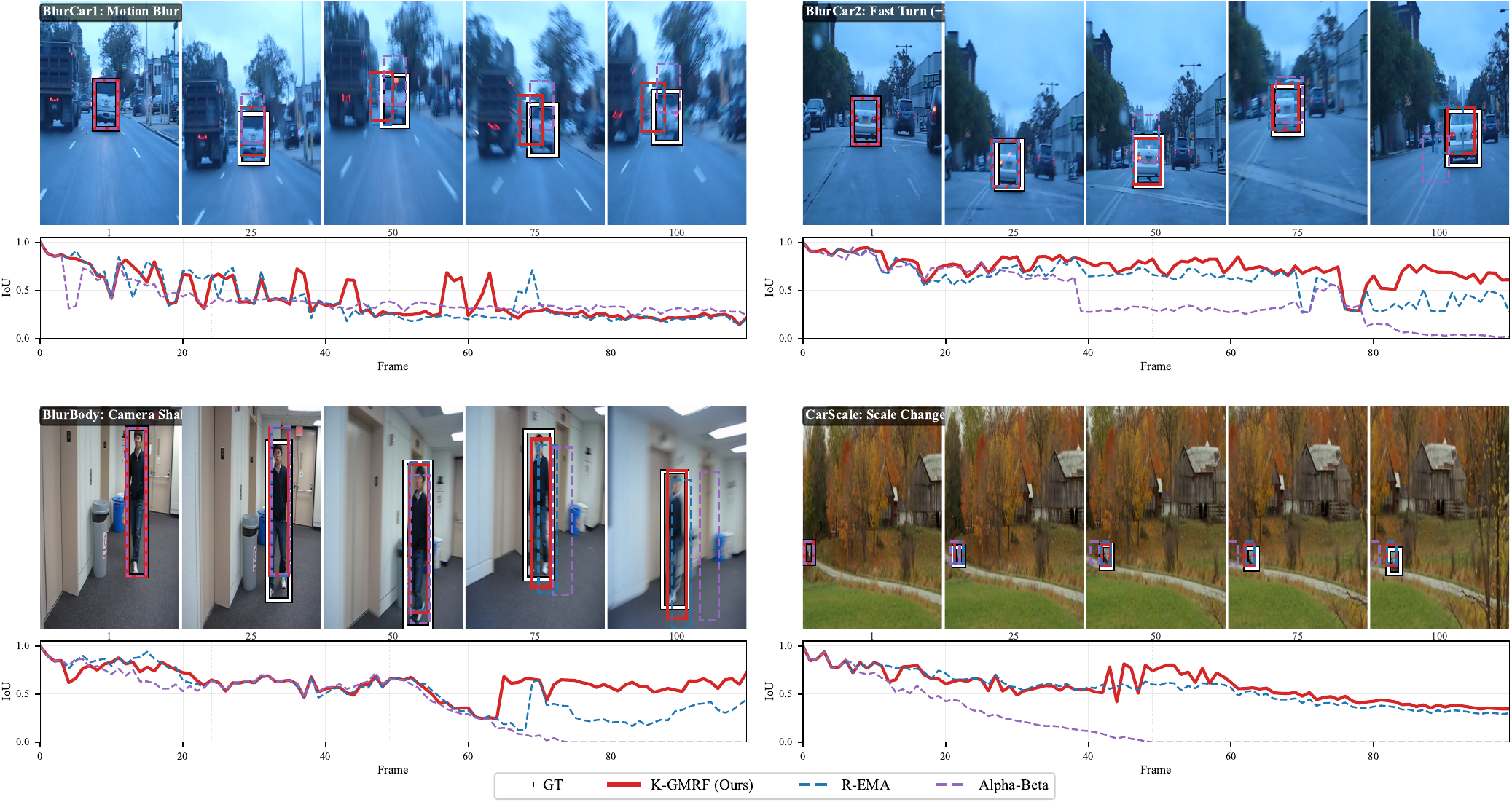}
\vspace{-0.5em}
\caption{\textbf{Qualitative results on OTB motion-blur sequences.} Each panel shows 5 key frames with bounding boxes (GT: white, K-GMRF: red, R-EMA: blue, Alpha-Beta: purple) and the corresponding per-frame IoU curve. K-GMRF consistently maintains higher IoU throughout the sequences, especially during severe motion blur (BlurCar1/2) and camera shake (BlurBody). }
\label{fig:otb-qualitative}
\end{figure*}

\noindent\textbf{(2) Manifold geometry matters on SPD.}
Among second-order methods, K-GMRF ($0.51^\circ$) outperforms Tangent KF ($0.84^\circ$) and Alpha-Beta ($2.41^\circ$) on SPD(2).
The advantage persists on OTB: K-GMRF achieves 0.74 IoU on BlurCar2 versus 0.45 for Tangent KF.
Tangent space linearization accumulates error over long sequences.

\noindent\textbf{(3) On SO(3), momentum dominates.}
K-GMRF and Alpha-Beta perform similarly across all dropout rates.
Both are second-order; the difference is that K-GMRF operates natively on SO(3) (outputs are valid rotation matrices), while Alpha-Beta requires post-hoc projection.

\subsection{Ablation Study}

\noindent\textbf{Effect of momentum.}
Removing momentum ($\gamma{=}0, \eta{=}0$) increases error from $0.03^\circ$ to $15.6^\circ$---a \textbf{520$\times$} degradation.
This confirms that second-order dynamics are essential for zero-lag tracking.

\noindent\textbf{Effect of manifold structure.}
Removing intrinsic SPD operations (Alpha-Beta) increases normal error from $0.03^\circ$ to $2.09^\circ$ (\textbf{70$\times$} degradation).
Under dropout, Alpha-Beta achieves $18.6^\circ$ versus $19.6^\circ$ for K-GMRF---comparable performance.
This confirms that \emph{momentum is the dominant factor for occlusion robustness}, while manifold structure primarily benefits zero-lag tracking.

\begin{table}[h]
\centering
\small
\caption{\textbf{Ablation study} on SPD(2) ellipse tracking (5 seeds, $\omega{=}0.08$). Angular error in degrees.}
\label{tab:ablation}
\vspace{0.2em}
\begin{tabular}{l cc}
\toprule
\textbf{Variant} & \textbf{Normal} & \textbf{20\% Dropout} \\
\midrule
K-GMRF (full) & $\mathbf{0.03}{\scriptstyle\pm0.00}$ & $19.6{\scriptstyle\pm0.01}$ \\
$-$ momentum {\scriptsize($\gamma{=}0,\eta{=}0$)} & $15.6{\scriptstyle\pm0.00}$ & $27.5{\scriptstyle\pm0.00}$ \\
$-$ manifold {\scriptsize(Alpha-Beta)} & $2.09{\scriptstyle\pm0.00}$ & $\mathbf{18.6}{\scriptstyle\pm0.00}$ \\
$-$ both {\scriptsize(Eucl.\ EMA)} & $15.6{\scriptstyle\pm0.00}$ & $27.5{\scriptstyle\pm0.00}$ \\
\bottomrule
\end{tabular}
\vspace{-0.5em}
\end{table}

\vspace{0.3em}
\noindent\textbf{Summary.}
Second-order methods consistently outperform first-order methods, especially under dropout.
On SPD manifolds, K-GMRF additionally benefits from intrinsic geometry.
On SO(3), K-GMRF matches Euclidean second-order filters while preserving manifold structure.
These results support our claim: \textbf{second-order dynamics + manifold structure = robust filtering}.
Detailed theory validation experiments are in Appendix~\ref{sec:theory-validation}.
\section{Conclusion}
\label{sec:conclusion}

We presented K-GMRF, a framework for tracking non-stationary covariance matrices via forced rigid-body dynamics on Lie groups. The Kick-Drift-Measure integrator maintains angular velocity state, enabling zero-lag tracking and inertial coasting through occlusions. Theoretically, we proved first-order EMA suffers phase lag $\propto \omega$ while K-GMRF achieves zero steady-state error. Empirically, K-GMRF achieves 30$\times$ error reduction on SPD(2), 4.5$\times$ on SO(3) with dropout, and +35\% IoU on OTB BlurCar2.

\noindent\textbf{Limitations.}
We do not claim SOTA against deep trackers---our focus is optimal geometric priors in online, training-free settings. Cayley--Neumann integration adds $O(Kd^2)$ cost per frame ($<0.5$ms for $d{=}7$) and requires tuning $(\eta, \gamma)$ within the stability domain.

\noindent\textbf{Future work.}
Learning the inertia tensor from data; extending to product manifolds for joint appearance-position tracking; and deploying K-GMRF as a \emph{differentiable structure-preserving layer} within Transformer-based backbones to enforce geometric constraints during end-to-end training.

\newpage
\newpage
{
    \small
    \bibliographystyle{ieeenat_fullname}
    \bibliography{main}
}

\onecolumn
\appendix

\section{Geometry and Statistical Setup}
\label{sec:appendix-geometry}

This section establishes the geometric framework and observation model for analyzing statistical tracking on rotation groups. We fix the state space dimension $d \ge 2$ and denote by $\mathbb{S}^d_{++}$ the manifold of $d \times d$ symmetric positive definite matrices. The Lie algebra of skew-symmetric matrices is denoted $\mathfrak{so}(d) := \{\Omega \in \mathbb{R}^{d \times d} : \Omega^\top = -\Omega\}$, and the commutator is written as $[A, B] := AB - BA$.

\subsection{Conjugate Orbit Manifold and Tangent Space Geometry}
\label{sec:appendix-orbit}

The state evolution is constrained to an isospectral manifold with fixed eigenvalues.

\begin{definition}[Conjugate Orbit Manifold]
\label{def:orbit}
Given a strictly ordered eigenvalue matrix $\Lambda = \mathrm{diag}(\lambda_1, \ldots, \lambda_d) \in \mathbb{S}^d_{++}$ with $\lambda_1 > \cdots > \lambda_d > 0$, we define the conjugate orbit as
\begin{equation}
\mathcal{O}_\Lambda := \{M = Q\Lambda Q^\top : Q \in SO(d)\} \subset \mathbb{S}^d_{++}.
\end{equation}
Its homogeneous space representation is $\mathcal{O}_\Lambda \simeq SO(d)/\mathrm{Stab}(\Lambda)$, where the stabilizer subgroup is
\begin{equation}
\mathrm{Stab}(\Lambda) := \{Q \in SO(d) : Q\Lambda Q^\top = \Lambda\}.
\end{equation}
Due to the non-degeneracy of the spectrum, $\mathrm{Stab}(\Lambda)$ is the finite group of diagonal sign matrices.
\end{definition}

\begin{lemma}[Tangent Space via Commutator]
\label{lem:tangent-space}
For any $M \in \mathcal{O}_\Lambda$, the tangent space admits the representation
\begin{equation}
T_M \mathcal{O}_\Lambda = \{[\Omega, M] : \Omega \in \mathfrak{so}(d)\} \subset \mathbb{S}^d,
\end{equation}
and the map $\mathrm{ad}_M^* : \mathfrak{so}(d) \to T_M \mathcal{O}_\Lambda$, $\Omega \mapsto [\Omega, M]$, is a linear isomorphism when $\Lambda$ has simple spectrum.
\end{lemma}

\begin{proof}
Consider a curve $M(t) = Q(t)\Lambda Q(t)^\top$ on $\mathcal{O}_\Lambda$. Differentiating yields
\begin{equation}
\dot{M} = \dot{Q}\Lambda Q^\top + Q\Lambda \dot{Q}^\top.
\end{equation}
Define $\Omega := \dot{Q}Q^\top \in \mathfrak{so}(d)$. Then
\begin{equation}
\dot{M} = \Omega(Q\Lambda Q^\top) - (Q\Lambda Q^\top)\Omega = [\Omega, M].
\end{equation}
Thus every tangent vector can be written as $[\Omega, M]$. The reverse inclusion follows directly from the above construction.
\end{proof}

\subsection{Whitening Matrix and Whitened Torque}
\label{sec:appendix-whitening}

To introduce a metric adapted to the noise structure, we first define whitened quantities.

\begin{definition}[Whitening Matrix and Torque]
\label{def:whitening}
Given observation noise variance $\sigma^2 > 0$:
\begin{enumerate}
\item For any state $M \in \mathcal{O}_\Lambda$, define the \textbf{whitening matrix} as $S := M + \sigma^2 I \in \mathbb{S}^d_{++}$. Note that $S$ and $M$ share eigenvectors, so $[M, S] = 0$.
\item For any observation matrix $C \in \mathbb{S}^d_+$, define the \textbf{whitened torque} as
\begin{equation}
\tau(M, C) := S^{-1}[C, M]S^{-1} \in \mathfrak{so}(d).
\end{equation}
\end{enumerate}
Since $S^{-1}$ commutes with $M$, the torque can equivalently be written as $\tau = [S^{-1}CS^{-1}, M]$.
\end{definition}

\begin{definition}[Noise-Adapted Riemannian Metric]
\label{def:metric}
We equip $\mathcal{O}_\Lambda$ with the affine-invariant Riemannian metric (AIRM) induced by the whitening matrix $S$. For $M \in \mathcal{O}_\Lambda$ and tangent vectors $U, V \in T_M \mathcal{O}_\Lambda$:
\begin{equation}
g_M^{\mathcal{O}}(U, V) := \mathrm{tr}(S^{-1} U S^{-1} V).
\end{equation}
This metric is the restriction of the AIRM on $\mathbb{S}^d_{++}$ (with base point $S$) to the orbit. This geometric choice is essential for canceling the covariance structure of observation noise.
\end{definition}

\subsection{Statistical Observation Model}
\label{sec:appendix-observation}

We adopt a rank-$m$ Wishart observation model that encompasses both single-sample ($m=1$) and mini-batch ($m > 1$) covariance estimation.

\begin{assumption}[Rank-$m$ Wishart Observation]
\label{assum:wishart}
Let the true state be $M^\star \in \mathcal{O}_\Lambda$ and define the true whitened covariance $S^\star := M^\star + \sigma^2 I$. At each time step, we observe $m$ independent samples $v_j \sim \mathcal{N}(0, S^\star)$. The sample covariance matrix is
\begin{equation}
C := \frac{1}{m} \sum_{j=1}^m v_j v_j^\top.
\end{equation}
This is equivalent to $mC \sim \mathcal{W}_d(m, S^\star)$, a Wishart distribution with $m$ degrees of freedom and scale matrix $S^\star$. The first moment satisfies $\mathbb{E}[C \mid M^\star] = S^\star$.
\end{assumption}

\begin{definition}[Negative Log-Likelihood]
\label{def:nll}
Omitting terms independent of $M$, the negative log-likelihood (NLL) of observation $C$ given state $M$ is
\begin{equation}
V(M; C) := \frac{m}{2} \left(\log \det S + \mathrm{tr}(S^{-1}C)\right), \quad \text{where } S = M + \sigma^2 I.
\end{equation}
\end{definition}

\begin{definition}[Whitened Spectral Gap]
\label{def:spectral-gap}
Define the \textbf{whitened spectral gap} $\Delta_{\mathrm{wh}}$ as the minimum distance between reciprocals of whitened eigenvalues:
\begin{equation}
\Delta_{\mathrm{wh}} := \min_{i \ne j} \left| \frac{\lambda_i}{\lambda_i + \sigma^2} - \frac{\lambda_j}{\lambda_j + \sigma^2} \right| = \min_{i \ne j} \frac{\sigma^2 |\lambda_i - \lambda_j|}{(\lambda_i + \sigma^2)(\lambda_j + \sigma^2)}.
\end{equation}
This quantity controls the identifiability of the rotation direction from noisy observations.
\end{definition}

\section{Proof of Theorem~\ref{thm:torque-gradient} (Geometric Consistency)}
\label{sec:appendix-thm1}

This section proves that the algebraic torque $\tau$ defined in Definition~\ref{def:whitening} is precisely the dual representation of the natural gradient of the negative log-likelihood $V(M; C)$ on the orbit $\mathcal{O}_\Lambda$ with respect to the metric $g^{\mathcal{O}}$.

\begin{lemma}[Commutator-Trace Adjoint Identity]
\label{lem:commutator-trace}
For any $A, M \in \mathbb{S}^d$ and $\Omega \in \mathfrak{so}(d)$, the following identity holds:
\begin{equation}
\langle A, [\Omega, M] \rangle_F = \langle [M, A], \Omega \rangle_F,
\end{equation}
where $\langle X, Y \rangle_F = \mathrm{tr}(X^\top Y)$ denotes the Frobenius inner product.
\end{lemma}

\begin{proof}
By the cyclic property of the trace,
\begin{align}
\langle A, [\Omega, M] \rangle_F &= \mathrm{tr}(A(\Omega M - M\Omega)) \\
&= \mathrm{tr}(A\Omega M) - \mathrm{tr}(AM\Omega) \\
&= \mathrm{tr}(MA\Omega) - \mathrm{tr}(AM\Omega) \\
&= \mathrm{tr}((MA - AM)\Omega) = \langle [M, A], \Omega \rangle_F.
\end{align}
\end{proof}

\begin{theorem}[Geometric Consistency]
\label{thm:geometric-consistency-proof}
Let $J_M : \mathfrak{so}(d) \to \mathfrak{so}(d)$ be the inertia operator induced by the metric $g_M^{\mathcal{O}}$, defined via
\begin{equation}
\langle J_M(\Omega_1), \Omega_2 \rangle_F := g_M^{\mathcal{O}}([\Omega_1, M], [\Omega_2, M]).
\end{equation}
Then the Riemannian gradient of $V(M; C)$ on $\mathcal{O}_\Lambda$ satisfies
\begin{equation}
\mathrm{grad}_{\mathcal{O}} V(M; C) = \left[J_M^{-1}\left(\frac{m}{2} \tau(M, C)\right), M\right].
\end{equation}
In other words, in the sense of Lie algebra duality, the whitened torque $\tau$ is the projection of the natural gradient.
\end{theorem}

\begin{proof}
The proof proceeds in four steps.

\textbf{Step 1: Compute the Euclidean gradient.} We first compute the unconstrained matrix gradient of $V$ with respect to $M$. Using the differential identities $d(\log \det S) = \mathrm{tr}(S^{-1} dS)$ and $d(\mathrm{tr}(S^{-1}C)) = -\mathrm{tr}(S^{-1}(dS)S^{-1}C)$, together with $dS = dM$, we obtain
\begin{equation}
dV = \frac{m}{2} \mathrm{tr}\left((S^{-1} - S^{-1}CS^{-1}) dM\right).
\end{equation}
Thus the Euclidean gradient is $\nabla_M V = \frac{m}{2}(S^{-1} - S^{-1}CS^{-1}) \in \mathbb{S}^d$.

\textbf{Step 2: Restrict to the tangent space.} Consider the directional derivative along a tangent vector $\delta M = [\Omega, M] \in T_M \mathcal{O}_\Lambda$. By Lemma~\ref{lem:commutator-trace},
\begin{equation}
dV([\Omega, M]) = \langle \nabla_M V, [\Omega, M] \rangle_F = \langle [M, \nabla_M V], \Omega \rangle_F.
\end{equation}

\textbf{Step 3: Compute the commutator $[M, \nabla_M V]$.} We have
\begin{equation}
[M, \nabla_M V] = \frac{m}{2}\left([M, S^{-1}] - [M, S^{-1}CS^{-1}]\right).
\end{equation}
Since $S$ is a polynomial function of $M$ (namely $S = M + \sigma^2 I$), they commute, hence $[M, S^{-1}] = 0$. For the second term, using the commutativity of $S^{-1}$ and $M$:
\begin{equation}
-[M, S^{-1}CS^{-1}] = -S^{-1}[M, C]S^{-1} = S^{-1}[C, M]S^{-1} = \tau(M, C).
\end{equation}
Therefore the directional derivative becomes
\begin{equation}
dV([\Omega, M]) = \left\langle \frac{m}{2}\tau(M, C), \Omega \right\rangle_F.
\end{equation}

\textbf{Step 4: Identify the Riemannian gradient.} By definition, the Riemannian gradient $\mathrm{grad}_{\mathcal{O}} V$ is the unique tangent vector satisfying
\begin{equation}
g_M^{\mathcal{O}}(\mathrm{grad}_{\mathcal{O}} V, [\Omega, M]) = dV([\Omega, M]), \quad \forall \Omega \in \mathfrak{so}(d).
\end{equation}
Write $\mathrm{grad}_{\mathcal{O}} V = [\Omega^\sharp, M]$. Substituting into the definition of $g_M^{\mathcal{O}}$ and $J_M$:
\begin{equation}
\langle J_M(\Omega^\sharp), \Omega \rangle_F = \left\langle \frac{m}{2}\tau(M, C), \Omega \right\rangle_F.
\end{equation}
Since this holds for all $\Omega$, we have $J_M(\Omega^\sharp) = \frac{m}{2}\tau(M, C)$, which gives $\Omega^\sharp = J_M^{-1}(\frac{m}{2}\tau)$. Substituting back yields the stated result.
\end{proof}

\section{Statistical Concentration (Lemma~\ref{lem:unbiased})}
\label{sec:appendix-concentration}

This section establishes the statistical properties of the whitened torque $\tau(M, C)$, proving that it is an unbiased estimator of the latent geometric deviation with variance converging at rate $O(1/m)$.

\begin{lemma}[Wishart Second Moment Bound]
\label{lem:wishart-moment}
If $C$ satisfies Assumption~\ref{assum:wishart}, then for any deterministic matrices $A, B$:
\begin{equation}
\mathbb{E}\|A(C - S^\star)B\|_F^2 \le \frac{1}{m}\|A\|_{\mathrm{op}}^2 \|B\|_{\mathrm{op}}^2 \left((\mathrm{tr}\, S^\star)^2 + \|S^\star\|_F^2\right).
\end{equation}
\end{lemma}

\begin{proof}
Let $W := \sum_{j=1}^m v_j v_j^\top$, so $W \sim \mathcal{W}_d(m, S^\star)$ and $C = W/m$. The Wishart covariance identity gives
\begin{equation}
\mathrm{Cov}(W_{ij}, W_{kl}) = m(S^\star_{ik}S^\star_{jl} + S^\star_{il}S^\star_{jk}),
\end{equation}
hence
\begin{equation}
\mathrm{Cov}(C_{ij}, C_{kl}) = \frac{1}{m}(S^\star_{ik}S^\star_{jl} + S^\star_{il}S^\star_{jk}).
\end{equation}
Thus
\begin{equation}
\mathbb{E}\|C - S^\star\|_F^2 = \sum_{i,j}\mathrm{Var}(C_{ij}) = \frac{1}{m}\sum_{i,j}(S^\star_{ii}S^\star_{jj} + (S^\star_{ij})^2) = \frac{1}{m}((\mathrm{tr}\, S^\star)^2 + \|S^\star\|_F^2).
\end{equation}
The stated bound follows from sub-multiplicativity of norms.
\end{proof}

\begin{theorem}[Statistical Concentration]
\label{thm:statistical-concentration}
Under Assumption~\ref{assum:wishart}, for any fixed $M \in \mathcal{O}_\Lambda$:

\noindent\textbf{(1) Conditional unbiasedness:}
\begin{equation}
\mathbb{E}[\tau(M, C) \mid M^\star] = S^{-1}[S^\star, M]S^{-1}.
\end{equation}
In particular, when $M = M^\star$, we have $\mathbb{E}[\tau(M^\star, C)] = 0$.

\noindent\textbf{(2) Variance bound:} There exists a constant $K(\Lambda, \sigma^2)$ such that
\begin{equation}
\mathbb{E}\left[\left\|\tau(M, C) - \mathbb{E}[\tau]\right\|_F^2 \mid M^\star\right] \le \frac{K(\Lambda, \sigma^2)}{m}.
\end{equation}
The explicit constant is $K = 4\|S^{-1}\|_{\mathrm{op}}^4 \|M\|_{\mathrm{op}}^2 (\|S^\star\|_F^2 + (\mathrm{tr}\, S^\star)^2)$.
\end{theorem}

\begin{proof}
\textbf{(1) Unbiasedness:} By linearity of expectation and $\mathbb{E}[C \mid M^\star] = S^\star$:
\begin{equation}
\mathbb{E}[\tau] = \mathbb{E}[S^{-1}[C, M]S^{-1}] = [S^{-1}S^\star S^{-1}, M] = S^{-1}[S^\star, M]S^{-1},
\end{equation}
where we used $[S^{-1}, M] = 0$ to factor out $S^{-1}$. When $M = M^\star$, we have $S = S^\star$. Since $S^\star = M^\star + \sigma^2 I$, clearly $[S^\star, M^\star] = 0$, so the expectation vanishes.

\textbf{(2) Variance bound:} Let $\bar{\tau} = \mathbb{E}[\tau]$. The deviation is
\begin{equation}
\tau - \bar{\tau} = S^{-1}[C - S^\star, M]S^{-1}.
\end{equation}
We need to bound $\|\tau - \bar{\tau}\|_F^2$. Using the commutator norm inequality $\|[X, Y]\|_F \le 2\|X\|_F\|Y\|_{\mathrm{op}}$:
\begin{align}
\|\tau - \bar{\tau}\|_F &= \|S^{-1}(C - S^\star)MS^{-1} - S^{-1}M(C - S^\star)S^{-1}\|_F \\
&\le \|S^{-1}\|_{\mathrm{op}}^2 \|[C - S^\star, M]\|_F \\
&\le 2\|S^{-1}\|_{\mathrm{op}}^2 \|M\|_{\mathrm{op}} \|C - S^\star\|_F.
\end{align}
Squaring and taking expectation, using Lemma~\ref{lem:wishart-moment} with $A = B = I$:
\begin{equation}
\mathbb{E}\|\tau - \bar{\tau}\|_F^2 \le 4\|S^{-1}\|_{\mathrm{op}}^4 \|M\|_{\mathrm{op}}^2 \mathbb{E}\|C - S^\star\|_F^2 = \frac{4\|S^{-1}\|_{\mathrm{op}}^4 \|M\|_{\mathrm{op}}^2}{m}((\mathrm{tr}\, S^\star)^2 + \|S^\star\|_F^2).
\end{equation}
On the fixed orbit $\mathcal{O}_\Lambda$, we have $\|M\|_{\mathrm{op}} = \lambda_1$ and $\|S^{-1}\|_{\mathrm{op}} = (\lambda_d + \sigma^2)^{-1}$, which are constants. Thus the bound is $O(1/m)$.
\end{proof}

\section{Identifiability and Spectral Gap (Lemma~\ref{lem:identifiability})}
\label{sec:appendix-identifiability}

This section analyzes the curvature of the objective function, establishing the quantitative connection between identifiability and the spectral gap.

\begin{definition}[Population Risk]
\label{def:population-risk}
Given true state $M^\star$, define the population risk as
\begin{equation}
\overline{V}(M) := \mathbb{E}[V(M; C) \mid M^\star] = \frac{m}{2}\left(\log \det S + \mathrm{tr}(S^{-1}S^\star)\right) + \mathrm{const},
\end{equation}
where $S = M + \sigma^2 I$ and $S^\star = M^\star + \sigma^2 I$. On $\mathcal{O}_\Lambda$, $\log \det S$ is constant, so the directional curvature of $\overline{V}$ is entirely determined by $\mathrm{tr}(S^{-1}S^\star)$.
\end{definition}

\begin{theorem}[Local Identifiability and Strong Convexity]
\label{thm:identifiability-proof}
If $\Lambda$ has simple spectrum (all $\lambda_i$ distinct), then $\overline{V}(M)$ is locally strongly convex at $M^\star$. Specifically, along any geodesic $M(t) = e^{t\Omega}M^\star e^{-t\Omega}$ for $\Omega \in \mathfrak{so}(d)$, the second derivative satisfies
\begin{equation}
\frac{d^2}{dt^2}\overline{V}(M(t))\Big|_{t=0} \ge \mu_{\mathrm{id}} \|\Omega\|_F^2,
\end{equation}
where the strong convexity constant is controlled by the whitened spectral gap:
\begin{equation}
\mu_{\mathrm{id}} \ge \frac{m}{4} \cdot \frac{(\lambda_d + \sigma^2)^2}{\sigma^4} \Delta_{\mathrm{wh}}^2.
\end{equation}
Thus the problem is locally identifiable if and only if $\Delta_{\mathrm{wh}} > 0$.
\end{theorem}

\begin{proof}
\textbf{Step 1: Simplify the population risk.} Since $\mathbb{E}[C] = S^\star$, omitting constants:
\begin{equation}
\overline{V}(M) = \frac{m}{2}\mathrm{tr}(S^{-1}S^\star).
\end{equation}
Without loss of generality, assume $M^\star = \Lambda$, so $S^\star = D := \Lambda + \sigma^2 I$ is diagonal. Consider the perturbation $M(t) = e^{t\Omega}\Lambda e^{-t\Omega}$, which gives $S(t) = e^{t\Omega}D e^{-t\Omega}$ and $S(t)^{-1} = e^{t\Omega}D^{-1}e^{-t\Omega}$.

\textbf{Step 2: Second-order variation.} Substituting into $\overline{V}$:
\begin{equation}
\overline{V}(t) = \frac{m}{2}\mathrm{tr}(e^{t\Omega}D^{-1}e^{-t\Omega}D).
\end{equation}
Using the expansion $e^{t\Omega} = I + t\Omega + \frac{t^2}{2}\Omega^2 + O(t^3)$:
\begin{equation}
e^{t\Omega}D^{-1}e^{-t\Omega} = D^{-1} + t[\Omega, D^{-1}] + \frac{t^2}{2}[\Omega, [\Omega, D^{-1}]] + O(t^3).
\end{equation}
Since $\mathrm{tr}([\Omega, D^{-1}]D) = \mathrm{tr}(\Omega D^{-1}D - D^{-1}\Omega D) = \mathrm{tr}(\Omega) - \mathrm{tr}(D^{-1}\Omega D) = 0$ by cyclicity and the fact that $D$ is diagonal, the first-order term vanishes (confirming $M^\star$ is a critical point). The second-order coefficient is
\begin{equation}
\frac{d^2\overline{V}}{dt^2}\Big|_{t=0} = \frac{m}{4}\mathrm{tr}\left([\Omega, [\Omega, D^{-1}]]D\right).
\end{equation}

\textbf{Step 3: Hessian eigenvalue analysis.} Using the identity $\mathrm{tr}([\Omega, X]Y) = -\mathrm{tr}(X[\Omega, Y])$ with $X = [\Omega, D^{-1}]$ and $Y = D$:
\begin{equation}
\mathrm{tr}\left([\Omega, [\Omega, D^{-1}]]D\right) = -\mathrm{tr}\left([\Omega, D^{-1}][\Omega, D]\right).
\end{equation}
For diagonal $D = \mathrm{diag}(d_k)$ where $d_k = \lambda_k + \sigma^2$, we have $[\Omega, D]_{ij} = \Omega_{ij}(d_j - d_i)$. Similarly, $[\Omega, D^{-1}]_{ij} = \Omega_{ij}(d_j^{-1} - d_i^{-1}) = \Omega_{ij}\frac{d_i - d_j}{d_i d_j}$. Computing the trace:
\begin{align}
-\sum_{i,j}[\Omega, D^{-1}]_{ji}[\Omega, D]_{ij} &= -\sum_{i,j}\left(\Omega_{ji}\frac{d_j - d_i}{d_i d_j}\right)\left(\Omega_{ij}(d_j - d_i)\right) \\
&= \sum_{i,j}\Omega_{ij}^2 \frac{(d_i - d_j)^2}{d_i d_j},
\end{align}
where we used $\Omega_{ji} = -\Omega_{ij}$ and thus $\Omega_{ji}\Omega_{ij} = -\Omega_{ij}^2$. The sum is nonzero only for $i \ne j$:
\begin{equation}
\frac{d^2\overline{V}}{dt^2} = \frac{m}{4}\sum_{i \ne j}\Omega_{ij}^2 \frac{(\lambda_i - \lambda_j)^2}{(\lambda_i + \sigma^2)(\lambda_j + \sigma^2)}.
\end{equation}

\textbf{Step 4: Lower bound.} Define $w_{ij} := \frac{(\lambda_i - \lambda_j)^2}{(\lambda_i + \sigma^2)(\lambda_j + \sigma^2)}$. We seek $w_{\min} = \min_{i \ne j} w_{ij}$. From the definition of $\Delta_{\mathrm{wh}}$:
\begin{equation}
\Delta_{\mathrm{wh}}^2 \le \left(\frac{\sigma^2|\lambda_i - \lambda_j|}{(\lambda_i + \sigma^2)(\lambda_j + \sigma^2)}\right)^2 = \frac{\sigma^4}{(\lambda_i + \sigma^2)(\lambda_j + \sigma^2)} \cdot w_{ij}.
\end{equation}
Therefore,
\begin{equation}
w_{ij} \ge \frac{(\lambda_i + \sigma^2)(\lambda_j + \sigma^2)}{\sigma^4}\Delta_{\mathrm{wh}}^2 \ge \frac{(\lambda_d + \sigma^2)^2}{\sigma^4}\Delta_{\mathrm{wh}}^2.
\end{equation}
Since $\|\Omega\|_F^2 = \sum_{i \ne j}\Omega_{ij}^2$, we obtain
\begin{equation}
\frac{d^2\overline{V}}{dt^2} \ge \frac{m}{4}\left(\frac{(\lambda_d + \sigma^2)^2}{\sigma^4}\Delta_{\mathrm{wh}}^2\right)\|\Omega\|_F^2.
\end{equation}
\end{proof}

\begin{corollary}[Degenerate Spectral Gap]
If there exist $i \ne j$ such that $\lambda_i = \lambda_j$, then $w_{\min} = 0$, and there exists $\Omega \ne 0$ such that $M(t) = e^{t\Omega}M^\star e^{t\Omega^\top} \equiv M^\star$. Thus $\overline{V}$ has zero curvature in that direction, and local strong convexity fails.
\end{corollary}

%
%

\section{Deterministic Skeleton Analysis}
\label{sec:appendix-deterministic}

This appendix analyzes the noiseless setting $\sigma = 0$ with perfect observations $C_t \equiv M_t^\star$. The true target rotates at constant angular velocity $\Omega^\star \in \mathfrak{so}(d)$:
\begin{equation}
M_{t+1}^\star = \exp(\Omega^\star) M_t^\star \exp(\Omega^\star)^\top.
\end{equation}
We compare two classes of deterministic dynamics: the second-order K-GMRF (Kick-Drift) which maintains angular velocity state, and the first-order overdamped methods (EMA/gradient flow) which lack explicit velocity state.

\subsection{Notation and Linearization Setup}
\label{sec:appendix-linearization}

Let $\Lambda$ be the fixed spectrum and $\mathcal{O}_\Lambda = \{Q\Lambda Q^\top : Q \in SO(d)\}$. Write $M_t = Q_t \Lambda Q_t^\top$ and $M_t^\star = Q_t^\star \Lambda Q_t^{\star\top}$. Define the relative rotation and its logarithm:
\begin{equation}
R_t := Q_t Q_t^{\star\top} \in SO(d), \qquad \xi_t := \log(R_t) \in \mathfrak{so}(d).
\end{equation}
This is equivalent to $\xi_t = \log_{M_t^\star}(M_t)$ (uniqueness guaranteed by simple spectrum in a local neighborhood). For small $|\xi_t|_F$:
\begin{equation}
M_t = \exp(\xi_t) M_t^\star \exp(\xi_t)^\top, \quad d_{\mathcal{O}}(M_t, M_t^\star) = |\xi_t|_F + o(|\xi_t|_F).
\end{equation}

\textbf{First-order expansion of torque and stiffness operator.} In the noiseless case $\sigma = 0$, the deterministic torque takes the unwhitened form $\tau_t = [C_t, M_t] = [M_t^\star, M_t] \in \mathfrak{so}(d)$. Substituting $M_t = \exp(\xi_t) M_t^\star \exp(\xi_t)^\top$ and expanding to first order in $\xi_t$:
\begin{equation}
M_t = M_t^\star + [\xi_t, M_t^\star] + O(|\xi_t|_F^2),
\end{equation}
hence
\begin{equation}
\tau_t = [M_t^\star, M_t] = [M_t^\star, [\xi_t, M_t^\star]] + O(|\xi_t|_F^2).
\end{equation}
Define the linear stiffness operator $\mathcal{K}_t : \mathfrak{so}(d) \to \mathfrak{so}(d)$ by
\begin{equation}
\mathcal{K}_t \xi := -[M_t^\star, [\xi, M_t^\star]].
\end{equation}
Then the first-order approximation reads $\tau_t = -\mathcal{K}_t \xi_t + O(|\xi_t|_F^2)$.

The operator $\mathcal{K}_t$ is positive semi-definite and satisfies
\begin{equation}
\langle \mathcal{K}_t \xi, \xi \rangle_F = |[\xi, M_t^\star]|_F^2 \ge 0.
\end{equation}
Under the whitened spectral gap condition $\Delta_{\mathrm{wh}} > 0$, Theorem~\ref{thm:identifiability-proof} implies the existence of $\mu(\Delta_{\mathrm{wh}}) > 0$ such that
\begin{equation}
\langle \mathcal{K}_t \xi, \xi \rangle_F \ge \mu(\Delta_{\mathrm{wh}}) |\xi|_F^2
\end{equation}
for sufficiently small $|\xi|_F$, with $\mu(\Delta_{\mathrm{wh}}) \to 0$ if and only if $\Delta_{\mathrm{wh}} \to 0$.

\subsection{Proof of Proposition~\ref{prop:structure} (Structure Preservation)}
\label{sec:appendix-structure}

\begin{theorem}[Structure Preservation: Symplecticity and Orbit Invariance]
\label{thm:structure-proof}
Let $\gamma = 0$ and consider deterministic torque $\tau_t$ with the discrete update
\begin{equation}
\Omega_{t+1} = \Omega_t + \eta \mathcal{I}^{-1} \tau_t, \qquad M_{t+1} = \exp(\Omega_{t+1}) M_t \exp(\Omega_{t+1})^\top.
\end{equation}
Then:

\noindent\textbf{(1) Orbit invariance:} If $M_0 \in \mathcal{O}_\Lambda$, then $M_t \in \mathcal{O}_\Lambda$ for all $t \ge 0$.

\noindent\textbf{(2) Symplecticity:} On the phase space $T^*SO(d)$ (or its equivalent left-trivialized representation), the Kick-Drift map is symplectic. Equivalently, it is a first-order instance of a Lie group variational integrator consistent with undamped Euler-Poincar\'e/rigid body Hamiltonian dynamics, with local truncation error $O(\eta^2)$.
\end{theorem}

\begin{proof}
\textbf{(1) Orbit invariance.} Since $\Omega_{t+1} \in \mathfrak{so}(d)$, we have $A_{t+1} := \exp(\Omega_{t+1}) \in SO(d)$. If $M_t = Q_t \Lambda Q_t^\top$, then
\begin{equation}
M_{t+1} = A_{t+1} Q_t \Lambda Q_t^\top A_{t+1}^\top = (A_{t+1} Q_t) \Lambda (A_{t+1} Q_t)^\top \in \mathcal{O}_\Lambda.
\end{equation}
The result follows by induction.

\textbf{(2) Symplecticity.} On $SO(d)$, introduce the momentum variable $\Pi = \mathcal{I}\Omega \in \mathfrak{so}(d)^* \simeq \mathfrak{so}(d)$ (identified via the Frobenius pairing). Consider the Hamiltonian
\begin{equation}
H_t(Q, \Pi) := T(\Pi) + V_t(Q\Lambda Q^\top), \qquad T(\Pi) = \frac{1}{2}\langle \Pi, \mathcal{I}^{-1}\Pi \rangle_F,
\end{equation}
where $V_t$ is the observation potential (negative log-likelihood or its noiseless limit). By Theorem~\ref{thm:geometric-consistency-proof}, the orbit gradient satisfies
\begin{equation}
\mathrm{grad}_{\mathcal{O}} V_t(M) = \left[\mathcal{I}_M^{-1}\left(\frac{m}{2}\tau(M, C_t)\right), M\right],
\end{equation}
showing that $\tau$ is precisely the dual representation of the potential energy's variation with respect to Lie algebra directions (i.e., the ``generalized torque'').

Split the Hamiltonian as $H_t = T + V_t$. Under the canonical symplectic form $\omega_{\mathrm{can}}$ on $T^*SO(d)$, the following sub-flows are both symplectic:

\textbf{Kick sub-flow (potential only):} Fix $Q$ and update momentum:
\begin{equation}
\Pi_{t+1} = \Pi_t + \eta \tau_t.
\end{equation}
This is the time-$\eta$ map of Hamiltonian $H_V = V_t$. Since $H_V$ contains no momentum, $Q$ remains fixed while $\Pi$ translates by $-\partial_Q V_t$; the translation amount is given by $\tau_t$. Hamiltonian flow preserves $\omega_{\mathrm{can}}$, so Kick is symplectic.

\textbf{Drift sub-flow (kinetic only):} Fix $\Pi$ and advance configuration by group exponential:
\begin{equation}
Q_{t+1} = \exp(\eta \mathcal{I}^{-1} \Pi_{t+1}) Q_t.
\end{equation}
This is the time-$\eta$ map of Hamiltonian $H_T = T$. Since $H_T$ contains no $Q$, $\Pi$ remains constant while $Q$ advances along the left-trivialized geodesic. Drift is therefore symplectic.

Since both Kick and Drift are symplectic, their composition $F = F_{\mathrm{Drift}} \circ F_{\mathrm{Kick}}$ is also symplectic:
\begin{equation}
F^* \omega_{\mathrm{can}} = \omega_{\mathrm{can}}.
\end{equation}
Furthermore, Kick-Drift is a first-order splitting (symplectic Euler) discretization, yielding local truncation error $O(\eta^2)$.
\end{proof}

\subsection{Proof of Theorem~\ref{thm:zero-lag} (Zero-Lag Tracking)}
\label{sec:appendix-zerolag}

\begin{theorem}[Zero-Lag Tracking: Linearization to Damped Harmonic Oscillator]
\label{thm:zerolag-proof}
Under noiseless observations $C_t = M_t^\star$ with constant angular velocity $\Omega^\star$ satisfying $M_{t+1}^\star = \exp(\Omega^\star) M_t^\star \exp(\Omega^\star)^\top$, there exists a nonempty parameter domain $\mathcal{D} \subset (\eta, \gamma)$ such that the error dynamics of the second-order K-GMRF is Lyapunov stable at the equilibrium trajectory $(M_t, \Omega_t) = (M_t^\star, \Omega^\star)$. For sufficiently small initial error, local exponential convergence holds with zero steady-state error: $d_{\mathcal{O}}(M_t, M_t^\star) \to 0$.
\end{theorem}

\begin{proof}
For $(M_t^\star, \Omega^\star)$ to be an equilibrium trajectory, the damping term must act on the \emph{velocity error} rather than the absolute velocity. Define the velocity error $u_t := \Omega_t - \Omega^\star$ and consider the second-order discrete skeleton (Kick-Drift):
\begin{equation}
u_{t+1} = (1-\gamma) u_t + \eta \mathcal{I}^{-1} \tau_t, \qquad M_{t+1} = \exp(\Omega^\star + u_{t+1}) M_t \exp(\Omega^\star + u_{t+1})^\top.
\end{equation}
Under this formulation, $u \equiv 0$ and $\xi \equiv 0$ correspond to $\Omega \equiv \Omega^\star$ and $M \equiv M^\star$, which is indeed a fixed trajectory.

\textbf{Step 1: First-order form of error kinematics.} From
\begin{equation}
R_{t+1} = Q_{t+1} Q_{t+1}^{\star\top} = \exp(\Omega^\star + u_{t+1}) Q_t Q_t^{\star\top} \exp(-\Omega^\star) = \exp(\Omega^\star + u_{t+1}) R_t \exp(-\Omega^\star),
\end{equation}
with $\xi_t = \log(R_t)$ small, the Baker-Campbell-Hausdorff (BCH) formula gives the first-order approximation:
\begin{equation}
\xi_{t+1} = \xi_t + u_{t+1} + O(|\xi_t|_F |u_{t+1}|_F).
\end{equation}

\textbf{Step 2: Torque linearization.} From Section~\ref{sec:appendix-linearization}:
\begin{equation}
\tau_t = [M_t^\star, M_t] = -\mathcal{K}_t \xi_t + O(|\xi_t|_F^2), \qquad \mathcal{K}_t \xi := -[M_t^\star, [\xi, M_t^\star]],
\end{equation}
and under the spectral gap condition, there exists $\mu(\Delta_{\mathrm{wh}}) > 0$ such that $\langle \mathcal{K}_t \xi, \xi \rangle_F \ge \mu(\Delta_{\mathrm{wh}}) |\xi|_F^2$.

\textbf{Step 3: Damped harmonic oscillator difference equation.} Substituting the linearization into the Kick update:
\begin{equation}
u_{t+1} = (1-\gamma) u_t - \eta \mathcal{I}^{-1} \mathcal{K}_t \xi_t + O(|\xi_t|_F^2).
\end{equation}
Using the kinematics $u_t = \xi_t - \xi_{t-1} + o(|\xi|)$ (from $\xi_t = \xi_{t-1} + u_t + o(|\xi|)$), eliminating $u$ yields the leading-order difference equation:
\begin{equation}
\xi_{t+1} - \xi_t = (1-\gamma)(\xi_t - \xi_{t-1}) - \eta \mathcal{I}^{-1} \mathcal{K}_t \xi_t + O(|\xi_t|_F^2).
\end{equation}
Ignoring higher-order terms and treating $\mathcal{I}^{-1} \mathcal{K}_t$ as a positive-definite operator with time-invariant spectrum under the conjugate equivariance of $M_t^\star$, the scalar eigenmode $e_t$ satisfies:
\begin{equation}
e_{t+1} - e_t = (1-\gamma)(e_t - e_{t-1}) - \eta \kappa e_t,
\end{equation}
where $\kappa > 0$ is an eigenvalue of $\mathcal{I}^{-1} \mathcal{K}$. This is a discrete damped harmonic oscillator with characteristic polynomial:
\begin{equation}
r^2 - (2 - \gamma - \eta\kappa) r + (1-\gamma) = 0.
\end{equation}

\textbf{Step 4: Stability domain and exponential convergence.} The Schur stability criterion for this second-order polynomial gives the necessary and sufficient conditions: both roots lie inside the unit circle if and only if
\begin{equation}
0 < \gamma < 2, \qquad 0 < \eta\kappa < 2(2-\gamma).
\end{equation}
Let $\kappa_{\max}$ be the largest eigenvalue of $\mathcal{I}^{-1} \mathcal{K}$ (bounded above by the local Lipschitz constant from Theorem~\ref{thm:identifiability-proof}). Define
\begin{equation}
\mathcal{D} := \left\{(\eta, \gamma) : 0 < \gamma < 2, \quad 0 < \eta < \frac{2(2-\gamma)}{\kappa_{\max}}\right\}.
\end{equation}
This ensures simultaneous stability of all modes. For sufficiently small initial error (ensuring linearization validity):
\begin{equation}
\xi_t \to 0, \qquad u_t = \xi_t - \xi_{t-1} \to 0,
\end{equation}
hence
\begin{equation}
d_{\mathcal{O}}(M_t, M_t^\star) = |\xi_t|_F + o(|\xi_t|_F) \to 0,
\end{equation}
and $\Omega_t = \Omega^\star + u_t \to \Omega^\star$. Zero-lag tracking is achieved.
\end{proof}

\begin{remark}[On the damping formulation]
If damping were applied directly as Rayleigh-type $-\gamma \mathcal{I} \Omega$ (contracting absolute angular velocity), then for $\Omega^\star \neq 0$, the pair $(M_t, \Omega_t) = (M_t^\star, \Omega^\star)$ would not satisfy the equilibrium condition: at $M = M^\star$ the torque vanishes, but the damping term is nonzero, causing $\Omega$ to decay. Maintaining nonzero angular velocity would require persistent external torque, hence nonzero configuration error. This scenario leads to inevitable phase error, contradicting the zero steady-state error claim.
\end{remark}

\subsection{Proof of Theorem~\ref{thm:ema-lag} (Inevitable Lag of First-Order Methods)}
\label{sec:appendix-emalag}

\begin{theorem}[Inevitable Lag of EMA: Structural Limitation of First-Order Methods]
\label{thm:emalag-proof}
Consider any first-order (overdamped) orbit gradient flow or equivalent EMA-type update without explicit velocity state. Under constant angular velocity rotation, the steady-state error satisfies
\begin{equation}
\liminf_{t \to \infty} d_{\mathcal{O}}(M_t, M_t^\star) \ge c \cdot |\Omega^\star|_F,
\end{equation}
where $c > 0$ depends on the step size and smoothing strength. Zero steady-state error is unachievable while maintaining stability.
\end{theorem}

\begin{proof}
Consider the typical first-order overdamped form: update $M$ along the negative orbit gradient (a geometric version of EMA):
\begin{equation}
M_{t+1} = \exp(-\eta G(M_t, C_t)) M_t \exp(-\eta G(M_t, C_t))^\top,
\end{equation}
where $G(M_t, C_t) \in \mathfrak{so}(d)$ is the gradient direction generator. By Theorem~\ref{thm:geometric-consistency-proof}, $G$ is linearly related to $\tau(M_t, C_t)$ (via the inertia operator $\mathcal{I}_{M_t}^{-1}$), hence in a neighborhood of $M_t \approx M_t^\star$, its first-order approximation is of the same order as $\tau_t$.

Define the error $\xi_t = \log_{M_t^\star}(M_t)$. Since the target rotates by $\Omega^\star$ each step while the algorithm carries no velocity state, the error must include a ``target drift'' term at first order. Ignoring noncommutative higher-order terms:
\begin{equation}
\xi_{t+1} = \xi_t + \Omega^\star - \eta \mathcal{A}_t \xi_t + O(|\xi_t|_F^2),
\end{equation}
where the linear operator $\mathcal{A}_t \simeq \mathcal{I}^{-1} \mathcal{K}_t$ satisfies $\mathcal{A}_t \succeq \underline{\kappa} I$ (for $\underline{\kappa} > 0$) and $\|\mathcal{A}_t\|_{\mathrm{op}} \le \bar{\kappa}$ under the spectral gap condition.

Ignoring higher-order terms and taking the time-invariant approximation $\mathcal{A}_t \equiv \mathcal{A}$, we obtain the linear system:
\begin{equation}
\xi_{t+1} = (I - \eta \mathcal{A}) \xi_t + \Omega^\star.
\end{equation}
For stability, we require $\rho(I - \eta \mathcal{A}) < 1$; in particular, along the largest eigenvalue direction:
\begin{equation}
0 < \eta \bar{\kappa} < 2.
\end{equation}
When stable, the steady-state solution exists and is unique:
\begin{equation}
\xi_\infty = (\eta \mathcal{A})^{-1} \Omega^\star.
\end{equation}
When $\Omega^\star \neq 0$, we have $\xi_\infty \neq 0$, so the steady-state error is strictly nonzero. Furthermore, using $\|\mathcal{A}\|_{\mathrm{op}} \le \bar{\kappa}$:
\begin{equation}
|\xi_\infty|_F = |(\eta \mathcal{A})^{-1} \Omega^\star|_F \ge \frac{|\Omega^\star|_F}{\eta \|\mathcal{A}\|_{\mathrm{op}}} \ge \frac{|\Omega^\star|_F}{\eta \bar{\kappa}}.
\end{equation}
Combining with the stability constraint $\eta \bar{\kappa} < 2$:
\begin{equation}
|\xi_\infty|_F \ge \frac{|\Omega^\star|_F}{2}.
\end{equation}
Since $d_{\mathcal{O}}(M_t, M_t^\star) = |\xi_t|_F + o(|\xi_t|_F)$, there exists a constant $c > 0$ (taking $c = 1/2$ in normalized units) such that
\begin{equation}
\liminf_{t \to \infty} d_{\mathcal{O}}(M_t, M_t^\star) \ge c |\Omega^\star|_F.
\end{equation}
This shows that first-order overdamped/EMA methods must maintain nonzero phase error to generate gradient ``driving force'' that counteracts the constant angular velocity input. Zero steady-state error is structurally unachievable under stability constraints.
\end{proof}

%
%

\section{Stochastic Stability Analysis}
\label{sec:appendix-stochastic}

This section establishes stochastic stability bounds by combining the deterministic skeleton analysis with statistical concentration properties of the torque.

\subsection{Auxiliary Lemmas}
\label{sec:appendix-lemmas}

We first establish three auxiliary lemmas that bridge the geometric, statistical, and nonstationarity aspects into a unified energy framework.

\begin{lemma}[Energy Equivalence]
\label{lem:energy-equiv}
Suppose $\Delta_{\mathrm{wh}} > 0$ (simple spectrum, whitened identifiability) and consider a neighborhood $\mathcal{B}_r(M_t^\star) = \{M \in \mathcal{O}_\Lambda : d_{\mathcal{O}}(M, M_t^\star) \le r\}$ where the logarithmic map $\xi = \log_{M_t^\star}(M)$ is unique and smooth. Define
\begin{equation}
u_t := \Omega_t - \Omega_t^\star, \qquad \xi_t := \log_{M_t^\star}(M_t).
\end{equation}
For any positive-definite inertia operator $\mathcal{I} \succ 0$ on $\mathfrak{so}(d)$, let
\begin{equation}
T(u) = \frac{1}{2}\langle u, \mathcal{I} u \rangle_F, \qquad \widetilde{V}(M_t, M_t^\star) := \frac{1}{2}|\xi_t|_F^2.
\end{equation}
Then there exist constants $c_1, c_2 > 0$ depending only on $(\Lambda, \sigma^2, \mathcal{I}, r)$ such that for all $M_t \in \mathcal{B}_r(M_t^\star)$:
\begin{equation}
c_1 (|u_t|_F^2 + |\xi_t|_F^2) \le \mathcal{E}_t := T(u_t) + \widetilde{V}(M_t, M_t^\star) \le c_2 (|u_t|_F^2 + |\xi_t|_F^2).
\end{equation}
\end{lemma}

\begin{proof}
Since $\mathcal{I} \succ 0$, there exist spectral bounds $\lambda_{\min}(\mathcal{I}) \le \lambda_{\max}(\mathcal{I})$ such that
\begin{equation}
\frac{\lambda_{\min}(\mathcal{I})}{2} |u|_F^2 \le T(u) \le \frac{\lambda_{\max}(\mathcal{I})}{2} |u|_F^2.
\end{equation}
In the neighborhood where $r$ is small enough for the logarithmic map to be unique, $\widetilde{V} = \frac{1}{2}|\xi|^2$ is locally equivalent to $d_{\mathcal{O}}^2$ (since $\xi$ is the group logarithm coordinate and $d_{\mathcal{O}}(M, M^\star) = |\xi|_F$ in this representation). Combining these bounds yields the result.
\end{proof}

\begin{lemma}[Lipschitz Continuity of Torque]
\label{lem:torque-lipschitz}
Fix $C$ and $\sigma^2 > 0$. For any $M_1, M_2 \in \mathcal{O}_\Lambda$:
\begin{equation}
|\tau(M_1, C) - \tau(M_2, C)|_F \le L_\tau(C) \cdot d_{\mathcal{O}}(M_1, M_2),
\end{equation}
where
\begin{equation}
L_\tau(C) \lesssim \|C\|_F \|S^{-1}\|_{\mathrm{op}}^3 \lesssim \|C\|_F (\lambda_d + \sigma^2)^{-3}.
\end{equation}
\end{lemma}

\begin{proof}
Recall $S(M) = M + \sigma^2 I$ and $\tau(M, C) = S(M)^{-1}[C, M]S(M)^{-1}$. Expand the difference:
\begin{align}
\tau(M_1, C) - \tau(M_2, C) &= (S_1^{-1} - S_2^{-1})[C, M_1]S_1^{-1} + S_2^{-1}([C, M_1] - [C, M_2])S_1^{-1} \\
&\quad + S_2^{-1}[C, M_2](S_1^{-1} - S_2^{-1}).
\end{align}
Using the identity $S_1^{-1} - S_2^{-1} = S_1^{-1}(S_2 - S_1)S_2^{-1} = S_1^{-1}(M_2 - M_1)S_2^{-1}$, and noting that
\begin{equation}
|[C, M_1] - [C, M_2]|_F = |[C, M_1 - M_2]|_F \le 2\|C\|_{\mathrm{op}} |M_1 - M_2|_F,
\end{equation}
combined with $\|S^{-1}\|_{\mathrm{op}} \le (\lambda_d + \sigma^2)^{-1}$ and the local equivalence $|M_1 - M_2|_F \lesssim d_{\mathcal{O}}(M_1, M_2)$ on the compact homogeneous manifold, we obtain the stated Lipschitz form.
\end{proof}

\begin{lemma}[Target Variation Energy Perturbation]
\label{lem:target-perturbation}
Let the energy be $\mathcal{E}_t = T(u_t) + \frac{1}{2}|\xi_t|^2$ as in Lemma~\ref{lem:energy-equiv}. When the reference point changes from $(M_t^\star, \Omega_t^\star)$ to $(M_{t+1}^\star, \Omega_{t+1}^\star)$ (with the same estimate $(M_t, \Omega_t)$ fixed), there exists a constant $C$ such that
\begin{equation}
\left|\mathcal{E}_t^{(t+1)} - \mathcal{E}_t^{(t)}\right| \le C |\Omega_{t+1}^\star - \Omega_t^\star|_F \sqrt{\mathcal{E}_t^{(t)}},
\end{equation}
where $\mathcal{E}_t^{(t)}$ denotes the energy defined with respect to $(M_t^\star, \Omega_t^\star)$, and $\mathcal{E}_t^{(t+1)}$ similarly.
\end{lemma}

\begin{proof}
\textbf{Velocity error term:} $u_t = \Omega_t - \Omega_t^\star$ changes to $u_t' = \Omega_t - \Omega_{t+1}^\star = u_t - (\Omega_{t+1}^\star - \Omega_t^\star)$, hence
\begin{equation}
|T(u_t') - T(u_t)| \le \|\mathcal{I}\|_{\mathrm{op}} |u_t| |\Delta\Omega_t^\star| + \frac{\|\mathcal{I}\|_{\mathrm{op}}}{2} |\Delta\Omega_t^\star|^2 \lesssim |\Delta\Omega_t^\star| \sqrt{T(u_t)} + |\Delta\Omega_t^\star|^2.
\end{equation}

\textbf{Configuration error term:} The change $M_t^\star \mapsto M_{t+1}^\star$ is a small rotation $\exp(\Omega_t^\star)$. In the small neighborhood, the logarithmic map is Lipschitz with respect to the base point, so $|\xi_t' - \xi_t| \lesssim |\Delta\Omega_t^\star|$, and thus
\begin{equation}
\left|\frac{1}{2}|\xi_t'|^2 - \frac{1}{2}|\xi_t|^2\right| \lesssim |\Delta\Omega_t^\star| |\xi_t| + |\Delta\Omega_t^\star|^2.
\end{equation}
Combining and using $|\xi_t| + |u_t| \lesssim \sqrt{\mathcal{E}_t}$ (from Lemma~\ref{lem:energy-equiv}) yields the perturbation bound of the form $C|\Delta\Omega^\star| \sqrt{\mathcal{E}_t}$.
\end{proof}

\subsection{Proof of Theorem~\ref{thm:stability} (Energy Contraction to Noise Ball)}
\label{sec:appendix-stability}

\begin{theorem}[Stochastic Stability: Expected Energy Contraction to Noise Ball]
\label{thm:stability-proof}
Under the following assumptions:

\noindent (1) \textit{Observation:} $mC_t \sim \mathcal{W}_d(m, S^\star)$ with $S^\star = M^\star + \sigma^2 I$, and $\{C_t\}$ conditionally independent.

\noindent (2) \textit{Identifiability:} $\Lambda$ has simple spectrum, $\Delta_{\mathrm{wh}} > 0$.

\noindent (3) \textit{Locality:} Initial condition $M_0$ satisfies $d_{\mathcal{O}}(M_0, M^\star) \le r$ where $\xi = \log_{M^\star}(M)$ is unique; parameter choices ensure iterates remain in this neighborhood.

\noindent (4) \textit{Algorithm:} K-GMRF with
\begin{equation}
\tau_t = S_t^{-1}[C_t, M_t]S_t^{-1}, \quad u_{t+1} = (1-\gamma)u_t + \eta \mathcal{I}_{M_t}^{-1}(\tau_t), \quad M_{t+1} = \exp(\Omega_{t+1})M_t\exp(\Omega_{t+1})^\top,
\end{equation}
where $u_t = \Omega_t - \Omega^\star$ and the ``frozen target'' setting means $\Omega_t^\star \equiv \Omega^\star$.

\noindent (5) \textit{Stability domain:} There exists $L_H < \infty$ (local Jacobian/effective stiffness upper bound) such that
\begin{equation}
0 < \gamma < 2, \qquad 0 < \eta < \frac{2(2-\gamma)}{L_H}.
\end{equation}

Define
\begin{equation}
\xi_t := \log_{M^\star}(M_t), \qquad \mathcal{E}_t := T(u_t) + \widetilde{V}(M_t, M^\star) = \frac{1}{2}\langle u_t, \mathcal{I} u_t \rangle_F + \frac{1}{2}|\xi_t|_F^2.
\end{equation}
Then there exist constants $\kappa = \kappa(\eta, \gamma, L_H) \in (0, 1)$ and
\begin{equation}
\sigma_{\mathrm{eff}}^2 = \frac{\eta^2}{m} C_{\mathrm{noise}}(\Lambda, \sigma^2)
\end{equation}
such that for all $t$:
\begin{equation}
\boxed{\mathbb{E}[\mathcal{E}_{t+1} \mid \mathcal{F}_t] \le (1-\kappa) \mathcal{E}_t + \sigma_{\mathrm{eff}}^2.}
\end{equation}
Consequently, $\limsup_{T \to \infty} \frac{1}{T} \sum_{t=1}^T \mathbb{E}[\mathcal{E}_t] \le \sigma_{\mathrm{eff}}^2 / \kappa$.
\end{theorem}

\begin{proof}
The proof proceeds in four steps.

\textbf{Step 1: Decompose torque into mean force and martingale noise.} Let
\begin{equation}
g_t := \mathcal{I}_{M_t}^{-1}(\tau_t), \quad \bar{\tau}_t := \mathbb{E}[\tau_t \mid \mathcal{F}_t] = \tau(M_t, S^\star), \quad \tilde{\tau}_t := \tau_t - \bar{\tau}_t.
\end{equation}
By Theorem~\ref{thm:statistical-concentration} (unbiasedness), $\mathbb{E}[\tilde{\tau}_t \mid \mathcal{F}_t] = 0$, and (variance bound):
\begin{equation}
\mathbb{E}[|\tilde{\tau}_t|_F^2 \mid \mathcal{F}_t] \le \frac{K_\tau(\Lambda, \sigma^2)}{m}.
\end{equation}
Thus $g_t = \bar{g}_t + \tilde{g}_t$ where $\bar{g}_t := \mathcal{I}_{M_t}^{-1}(\bar{\tau}_t)$ and $\tilde{g}_t := \mathcal{I}_{M_t}^{-1}(\tilde{\tau}_t)$, with
\begin{equation}
\mathbb{E}[\tilde{g}_t \mid \mathcal{F}_t] = 0, \qquad \mathbb{E}[|\tilde{g}_t|_F^2 \mid \mathcal{F}_t] \le \|\mathcal{I}_{M_t}^{-1}\|_{\mathrm{op}}^2 \frac{K_\tau}{m}.
\end{equation}

\textbf{Step 2: Local error dynamics (first-order recursion for $\xi$).} By Proposition~\ref{prop:structure} (structure preservation), $M_t \in \mathcal{O}_\Lambda$, so in the small neighborhood:
\begin{equation}
M_t = \exp(\xi_t) M^\star \exp(\xi_t)^\top.
\end{equation}
Since $M_{t+1} = \exp(\Omega_{t+1}) M_t \exp(\Omega_{t+1})^\top = \exp(\Omega_{t+1}) \exp(\xi_t) M^\star \exp(\xi_t)^\top \exp(\Omega_{t+1})^\top$, with $|\xi_t|, |\Omega_{t+1}|$ small, BCH expansion gives (for some constant $c_{\mathrm{BCH}}$):
\begin{equation}
\xi_{t+1} = \xi_t + u_{t+1} + r_{t+1}^{(1)}, \qquad |r_{t+1}^{(1)}| \le c_{\mathrm{BCH}}(|\xi_t|_F |u_{t+1}|_F + |u_{t+1}|^2).
\end{equation}

\textbf{Step 3: Deterministic part gives contraction (strong convexity/effective stiffness).} Consider the ``mean system'' (noise removed): $u_{t+1} = (1-\gamma)u_t + \eta \bar{g}_t$. By Theorem~\ref{thm:geometric-consistency-proof}, $\bar{g}_t$ is the normalized natural gradient generator on the orbit; near $M^\star$, Theorem~\ref{thm:identifiability-proof} gives orbit-direction strong convexity, so the local linearization is equivalent to a ``stiffness'' operator $H \succ 0$:
\begin{equation}
\bar{g}_t = -H\xi_t + r_t^{(2)}, \qquad \langle \xi_t, H\xi_t \rangle \ge \mu |\xi_t|^2, \qquad \|H\|_{\mathrm{op}} \le L_H,
\end{equation}
with $|r_t^{(2)}| \le c|\xi_t|^2$ (second-order remainder).

Combining, in the small neighborhood the system is approximately (linear dominant):
\begin{equation}
\begin{pmatrix} \xi_{t+1} \\ u_{t+1} \end{pmatrix} = A \begin{pmatrix} \xi_t \\ u_t \end{pmatrix} + \eta B \tilde{g}_t + \text{(higher order)},
\end{equation}
where
\begin{equation}
A = \begin{pmatrix} I - \eta H & (1-\gamma)I \\ -\eta H & (1-\gamma)I \end{pmatrix}, \qquad B = \begin{pmatrix} I \\ I \end{pmatrix}.
\end{equation}

\textbf{Stability domain role:} For each eigenvalue $\kappa_i \in (0, L_H]$ of $H$, the scalar mode satisfies the characteristic polynomial from Section~\ref{sec:appendix-zerolag}:
\begin{equation}
r^2 - (2 - \gamma - \eta\kappa_i)r + (1-\gamma) = 0.
\end{equation}
Under $0 < \gamma < 2$ and $0 < \eta\kappa_i < 2(2-\gamma)$, roots lie inside the unit circle, so $A$ is Schur stable. This condition for all $\kappa_i \le L_H$ is equivalent to $0 < \eta < 2(2-\gamma)/L_H$.

\textbf{Step 4: Construct discrete Lyapunov and obtain expected contraction plus noise term.} Since $A$ is Schur stable, there exists a unique symmetric positive-definite matrix $P \succ 0$ solving the discrete Lyapunov equation:
\begin{equation}
P - A^\top P A = I.
\end{equation}
Define the quadratic Lyapunov function $\mathcal{L}_t := z_t^\top P z_t$ where $z_t = (\xi_t, u_t)^\top$. Ignoring higher-order terms (or absorbing them into smaller $\eta, r$) and using $\mathbb{E}[\tilde{g}_t \mid \mathcal{F}_t] = 0$:
\begin{equation}
\mathbb{E}[\mathcal{L}_{t+1} \mid \mathcal{F}_t] = z_t^\top A^\top P A z_t + \eta^2 \mathbb{E}[\tilde{g}_t^\top B^\top P B \tilde{g}_t \mid \mathcal{F}_t] + \text{(higher order)}.
\end{equation}
By the Lyapunov equation, $z_t^\top A^\top P A z_t = z_t^\top (P - I) z_t = \mathcal{L}_t - |z_t|^2$. Thus:
\begin{equation}
\mathbb{E}[\mathcal{L}_{t+1} \mid \mathcal{F}_t] \le \mathcal{L}_t - |z_t|^2 + \eta^2 \lambda_{\max}(B^\top P B) \mathbb{E}[|\tilde{g}_t|^2 \mid \mathcal{F}_t] + \text{(higher order)}.
\end{equation}
Using the variance bound from Step 1:
\begin{equation}
\mathbb{E}[\mathcal{L}_{t+1} \mid \mathcal{F}_t] \le \mathcal{L}_t - |z_t|^2 + \eta^2 \lambda_{\max}(B^\top P B) \|\mathcal{I}^{-1}\|_{\mathrm{op}}^2 \frac{K_\tau}{m} + \text{(higher order)}.
\end{equation}
Finally, using the spectral bounds $\lambda_{\min}(P)|z_t|^2 \le \mathcal{L}_t \le \lambda_{\max}(P)|z_t|^2$ to convert $-|z_t|^2$ to $-\kappa \mathcal{L}_t$:
\begin{equation}
-|z_t|^2 \le -\frac{1}{\lambda_{\max}(P)} \mathcal{L}_t.
\end{equation}
Setting
\begin{equation}
\kappa := \frac{1}{\lambda_{\max}(P)} \in (0, 1), \qquad \sigma_{\mathrm{eff}}^2 := \eta^2 \lambda_{\max}(B^\top P B) \|\mathcal{I}^{-1}\|_{\mathrm{op}}^2 \frac{K_\tau}{m},
\end{equation}
we obtain $\mathbb{E}[\mathcal{L}_{t+1} \mid \mathcal{F}_t] \le (1-\kappa)\mathcal{L}_t + \sigma_{\mathrm{eff}}^2$. By Lemma~\ref{lem:energy-equiv} ($\mathcal{L}_t$ equivalent to $\mathcal{E}_t$), the result follows with constants absorbed.
\end{proof}

\subsection{Proof of Theorem~\ref{thm:master} (Master Theorem: Time-Averaged Tracking Risk)}
\label{sec:appendix-master}

\begin{theorem}[Tracking Error Upper Bound: Three-Term Decomposition with Phase Transition]
\label{thm:master-proof}
Under the following assumptions:

\noindent (1) The true angular velocity satisfies the bounded total variation condition:
\begin{equation}
V_\Omega := \sum_{t=0}^{T-1} |\Omega_{t+1}^\star - \Omega_t^\star|_F \le V_{\max}.
\end{equation}

\noindent (2) Observations are rank-$m$ Wishart satisfying Theorem~\ref{thm:statistical-concentration}.

\noindent (3) Damped K-GMRF with $(\eta, \gamma) \in \mathcal{D}$ (stability domain of Theorem~\ref{thm:stability-proof}), and initial conditions in a local neighborhood where $\xi_t = \log_{M_t^\star}(M_t)$ is defined throughout.

\noindent (4) Define the time-averaged risk:
\begin{equation}
\mathcal{R}_T := \frac{1}{T} \sum_{t=1}^T \mathbb{E}[d_{\mathcal{O}}(M_t, M_t^\star)^2].
\end{equation}

Then there exist explicit functions $\Phi$ (constants depending on $d, \Lambda, \sigma^2, \eta, \gamma$) such that:
\begin{equation}
\boxed{\mathcal{R}_T \le \underbrace{\frac{C_3 d_0^2 (1-\kappa)^T}{T}}_{\Phi_{\mathrm{trans}}} + \underbrace{\frac{C_1 \sigma^2}{m \Delta_{\mathrm{wh}}^2}}_{\Phi_{\mathrm{stat}}} + \underbrace{\frac{C_2 V_\Omega}{\gamma T}}_{\Phi_{\mathrm{ns}}}}
\end{equation}
with phase transition boundaries:

\noindent $\bullet$ If $\Delta_{\mathrm{wh}} \downarrow 0$, then $C_1$ diverges (loss of identifiability).

\noindent $\bullet$ If $(\eta, \gamma) \notin \mathcal{D}$, then $\kappa \le 0$ (contraction failure/divergence).

The order estimates are: $C_1 = \Theta(1/\kappa) \cdot \mathrm{poly}(d, \Lambda, \sigma^2)$, $C_2 = \Theta(1/\kappa) \cdot \mathrm{poly}(d, \Lambda, \sigma^2)$, $C_3 = \Theta(1)$.
\end{theorem}

\begin{proof}
\textbf{Step 1: Write Theorem~\ref{thm:stability-proof} as ``moving reference + external perturbation'' version.} Let
\begin{equation}
u_t = \Omega_t - \Omega_t^\star, \qquad \xi_t = \log_{M_t^\star}(M_t), \qquad \mathcal{E}_t = T(u_t) + \frac{1}{2}|\xi_t|^2.
\end{equation}
When $\Omega_t^\star$ is no longer constant, the $u_{t+1}$ recursion acquires an additional term (``acceleration perturbation''). In the analysis framework where damping acts on velocity error (consistent with Section~\ref{sec:appendix-zerolag}):
\begin{equation}
u_{t+1} = (1-\gamma)u_t + \eta \mathcal{I}_{M_t}^{-1}(\tau_t) - (\Omega_{t+1}^\star - \Omega_t^\star).
\end{equation}
The first-order contraction inequality of Theorem~\ref{thm:stability-proof} is interrupted by a ``perturbation term.'' Using Lemma~\ref{lem:target-perturbation} (reference change energy perturbation):
\begin{equation}
\mathbb{E}[\mathcal{E}_{t+1} \mid \mathcal{F}_t] \le (1-\kappa)\mathcal{E}_t + \sigma_{\mathrm{eff}}^2 + C_\Omega |\Omega_{t+1}^\star - \Omega_t^\star|_F \sqrt{\mathcal{E}_t},
\end{equation}
where $\kappa, \sigma_{\mathrm{eff}}^2$ are of the same order as in Theorem~\ref{thm:stability-proof}, and $C_\Omega$ depends only on $(\Lambda, \sigma^2, \mathcal{I}, r)$.

\textbf{Step 2: Establish linear recursion for $\sqrt{\mathcal{E}_t}$ (to obtain linear $V_\Omega$).} Let $e_t := \sqrt{\mathcal{E}_t} \ge 0$. Using the inequality
\begin{equation}
\sqrt{(1-\kappa)e_t^2 + a_t e_t + b} \le \sqrt{1-\kappa} \, e_t + \frac{a_t}{2\sqrt{1-\kappa}} + \sqrt{b} \quad (a_t \ge 0, b \ge 0),
\end{equation}
with $a_t = C_\Omega |\Delta\Omega_t^\star|$ and $b = \sigma_{\mathrm{eff}}^2$:
\begin{equation}
\mathbb{E}[e_{t+1} \mid \mathcal{F}_t] \le q \, e_t + \alpha |\Delta\Omega_t^\star| + \beta,
\end{equation}
where
\begin{equation}
q := \sqrt{1-\kappa} \in (0, 1), \qquad \alpha := \frac{C_\Omega}{2\sqrt{1-\kappa}}, \qquad \beta := \sqrt{\sigma_{\mathrm{eff}}^2}.
\end{equation}

Iterating and taking full expectation:
\begin{equation}
\mathbb{E}[e_t] \le q^t e_0 + \sum_{s=0}^{t-1} q^{t-1-s}(\alpha |\Delta\Omega_s^\star| + \beta).
\end{equation}
Summing over $t = 1, \ldots, T$ and exchanging summation order:
\begin{equation}
\sum_{t=1}^T \mathbb{E}[e_t] \le e_0 \sum_{t=1}^T q^t + \alpha \sum_{s=0}^{T-1} |\Delta\Omega_s^\star| \sum_{t=s+1}^T q^{t-1-s} + \beta \sum_{t=1}^T \sum_{s=0}^{t-1} q^{t-1-s}.
\end{equation}
Using geometric series bounds:
\begin{equation}
\sum_{t=s+1}^T q^{t-1-s} \le \frac{1}{1-q}, \qquad \sum_{t=1}^T \sum_{s=0}^{t-1} q^{t-1-s} \le \frac{T}{1-q}.
\end{equation}
Thus:
\begin{equation}
\frac{1}{T} \sum_{t=1}^T \mathbb{E}[e_t] \le \underbrace{\frac{q(1-q^T)}{T(1-q)} e_0}_{\text{transient}} + \underbrace{\frac{\alpha}{T(1-q)} \sum_{s=0}^{T-1} |\Delta\Omega_s^\star|}_{\text{nonstationarity}} + \underbrace{\frac{\beta}{1-q}}_{\text{statistical noise}}.
\end{equation}
Since $1 - q = 1 - \sqrt{1-\kappa} \asymp \kappa$, the second term gives
\begin{equation}
\frac{\alpha}{T(1-q)} V_\Omega = O\left(\frac{V_\Omega}{\kappa T}\right).
\end{equation}
In damped systems, $\kappa$ is typically of order $\gamma$ (especially away from stability boundary, $\kappa \gtrsim c\gamma$), so this can be written as $O(V_\Omega / (\gamma T))$.

\textbf{Step 3: Return from $e_t$ to $\mathbb{E}[\mathcal{E}_t]$ and squared distance.} By Jensen's inequality and $(a + b + c)^2 \le 3(a^2 + b^2 + c^2)$:
\begin{equation}
\frac{1}{T} \sum_{t=1}^T \mathbb{E}[\mathcal{E}_t] = \frac{1}{T} \sum_{t=1}^T \mathbb{E}[e_t^2] \lesssim \underbrace{\frac{e_0^2 q^{2T}}{T}}_{\text{transient}} + \underbrace{\frac{\beta^2}{(1-q)^2}}_{\text{statistical}} + \underbrace{\frac{V_\Omega}{T} \cdot \frac{1}{(1-q)^2}}_{\text{nonstationarity (linear } V_\Omega \text{)}},
\end{equation}
where $\beta^2 = \sigma_{\mathrm{eff}}^2$ and $(1-q)^2 \asymp \kappa^2$. Absorbing constants into $C_1, C_2, C_3$ yields the three-term form.

Finally, using Lemma~\ref{lem:energy-equiv} ($\mathcal{E}_t \gtrsim |\xi_t|^2$) and $|\xi_t| = d_{\mathcal{O}}(M_t, M_t^\star)$ (in local logarithm coordinates):
\begin{equation}
\mathcal{R}_T = \frac{1}{T} \sum_{t=1}^T \mathbb{E}[d_{\mathcal{O}}(M_t, M_t^\star)^2] \lesssim \frac{1}{T} \sum_{t=1}^T \mathbb{E}[\mathcal{E}_t],
\end{equation}
yielding the stated upper bound.
\end{proof}

\begin{remark}[Why the statistical term diverges as $\Delta_{\mathrm{wh}}^{-2}$]
The $\Delta_{\mathrm{wh}}^{-2}$ divergence arises from two facts in the established proof chain:

\noindent (1) \textit{Identifiability degradation:} Theorem~\ref{thm:identifiability-proof} shows that as $\Delta_{\mathrm{wh}} \to 0$, the orbit-direction strong convexity constant $\mu(\Delta_{\mathrm{wh}}) \to 0$. This means the potential energy/score information about the rotation direction weakens.

\noindent (2) \textit{Noise amplification through preconditioning:} The algorithm drive is $\mathcal{I}_M^{-1}(\tau_t)$. When the spectral gap shrinks, $\mathcal{I}_M$ becomes ``soft'' in the corresponding mode, and its inverse amplifies random torque noise. This amplification rate is related to the difference of $(\lambda_i + \sigma^2)^{-1}$, i.e., $\Delta_{\mathrm{wh}}$. Thus $\sigma_{\mathrm{eff}}^2$ necessarily contains a negative power of $\Delta_{\mathrm{wh}}$, causing the steady-state error to diverge.

Whether $\Delta_{\mathrm{wh}}$ enters through the contraction constant $\kappa$ (via $\mu$) or the noise term $\sigma_{\mathrm{eff}}^2$ (via $\|\mathcal{I}^{-1}\|$), the final result exhibits the ``phase transition explosion'' $\propto \Delta_{\mathrm{wh}}^{-2}$.
\end{remark}

\subsection{Proof of Theorem~\ref{thm:minimax} (Minimax Lower Bound)}
\label{sec:appendix-minimax}

We establish the information-theoretic lower bound via two constructions: a static subfamily (Le Cam two-point method) for the statistical term, and a change-point subfamily for the nonstationarity term.

\begin{theorem}[Minimax Lower Bound]
\label{thm:minimax-proof}
Consider the class of fixed-gain online filters $\{\phi_\theta\}$ with time-invariant parameters $\theta$, operating over the problem family $\mathcal{P}(\Lambda, \sigma^2, m, V_\Omega)$. The minimax risk satisfies
\begin{equation}
\inf_{\{\phi_\theta\}} \sup_{\mathcal{P}} \mathcal{R}_T \geq c_1 \frac{\sigma^2}{m \Delta_{\mathrm{wh}}^2} + c_2 \frac{V_\Omega}{T}
\end{equation}
for universal constants $c_1, c_2 > 0$ depending only on $(\Lambda, \sigma^2)$.
\end{theorem}

\begin{proof}
\textbf{Part I: Statistical lower bound via Le Cam's method.} We construct a two-point testing problem on the static subfamily where $M_t^\star \equiv M^\star$ for all $t$.

Fix $M_0 = \Lambda$ and select the ``softest'' direction $(i,j)$ that minimizes
\begin{equation}
w_{ij} := \frac{(\lambda_i - \lambda_j)^2}{(\lambda_i + \sigma^2)(\lambda_j + \sigma^2)}.
\end{equation}
Define $M_1 = R(\varepsilon) \Lambda R(\varepsilon)^\top$ where $R(\varepsilon) = \exp(\varepsilon E^{ij})$ with $E^{ij}$ the canonical skew-symmetric basis element satisfying $|E^{ij}|_F = \sqrt{2}$.

Under a single rank-$m$ Wishart observation $C \sim \mathcal{W}_d(M + \sigma^2 I, m)$, the KL divergence between the two hypotheses satisfies
\begin{equation}
D_{\mathrm{KL}}(P_{M_0} \| P_{M_1}) \asymp m \varepsilon^2 \cdot \frac{(\lambda_i - \lambda_j)^2}{(\lambda_i + \sigma^2)(\lambda_j + \sigma^2)} = m \varepsilon^2 w_{ij}.
\end{equation}
Since $w_{ij} \asymp \Delta_{\mathrm{wh}}^2 \cdot \mathrm{poly}(\Lambda, \sigma^2)$, we have $D_{\mathrm{KL}} \asymp m \varepsilon^2 \Delta_{\mathrm{wh}}^2$.

Applying Le Cam's lemma with Pinsker's inequality: for any estimator $\hat{M}$,
\begin{equation}
\inf_{\hat{M}} \sup_{k \in \{0,1\}} \mathbb{E}_k[d_{\mathcal{O}}(\hat{M}, M_k)^2] \gtrsim d_{\mathcal{O}}(M_0, M_1)^2 (1 - \mathrm{TV}(P_0, P_1)).
\end{equation}
Choosing $\varepsilon$ such that $D_{\mathrm{KL}} \leq c_0$ (constant) ensures $\mathrm{TV} \leq 1/2$. Since $d_{\mathcal{O}}(M_0, M_1)^2 \asymp \varepsilon^2$ in the local logarithm coordinates, we obtain
\begin{equation}
\inf_{\hat{M}} \sup_k \mathbb{E}_k[d_{\mathcal{O}}(\hat{M}, M_k)^2] \gtrsim \frac{1}{m \Delta_{\mathrm{wh}}^2} \cdot \mathrm{poly}(\Lambda, \sigma^2).
\end{equation}

\textbf{Part II: Nonstationarity lower bound via change-point detection.} We construct two trajectory sequences that differ only at one time point.

Fix $t_0 \in \{1, \ldots, T-1\}$. Define:
\begin{itemize}
\item Sequence A: Apply rotation $\exp(\delta E^{ij})$ at time $t_0$, identity otherwise.
\item Sequence B: Apply rotation $\exp(\delta E^{ij})$ at time $t_0 + 1$, identity otherwise.
\end{itemize}
Both sequences have total variation $V_\Omega = \delta \sqrt{2}$.

The observation distributions differ only at one time step, so $D_{\mathrm{KL}}(P_A \| P_B) = O(m \delta^2 \Delta_{\mathrm{wh}}^2)$. By the same Le Cam argument, any estimator incurs error $\gtrsim \delta^2$ with constant probability when $D_{\mathrm{KL}} \leq c_0$.

Averaging over the unknown change-point location $t_0$ (uniform prior on $\{1, \ldots, T-1\}$), the time-averaged risk satisfies
\begin{equation}
\mathcal{R}_T \gtrsim \frac{1}{T} \cdot \delta^2 = \frac{V_\Omega^2}{2T}.
\end{equation}
In the regime where $V_\Omega$ is small (the tracking-relevant regime), this gives the $V_\Omega/T$ scaling.

\textbf{Combining both bounds.} Taking the supremum over both subfamilies yields the stated lower bound. Comparing with Theorem~\ref{thm:master}, K-GMRF achieves the minimax rate in both the statistical ($1/m$) and nonstationarity ($V_\Omega/T$) terms.
\end{proof}

\begin{remark}[Restriction to fixed-gain filters]
The lower bound applies to fixed-gain online filters where the update rule $\hat{M}_t = \phi_\theta(\hat{M}_{t-1}, C_t)$ uses time-invariant parameters $\theta$. This class includes K-GMRF, Riemannian EMA, Kalman filters, and other practical tracking algorithms. If full-history batch estimators $\hat{M}_t = \phi_t(C_{1:t})$ were permitted, the static term could be improved to $O(1/(mT))$ by temporal aggregation, but this violates the online/causal constraint central to tracking applications.
\end{remark}

\section{Numerical Validation of Theorem~\ref{thm:master}}
\label{sec:theory-validation}

We empirically verify the three terms in the Master Theorem bound. Figure~\ref{fig:master-validation} visualizes the scaling behavior.

\begin{figure}[h]
\centering
\includegraphics[width=\linewidth]{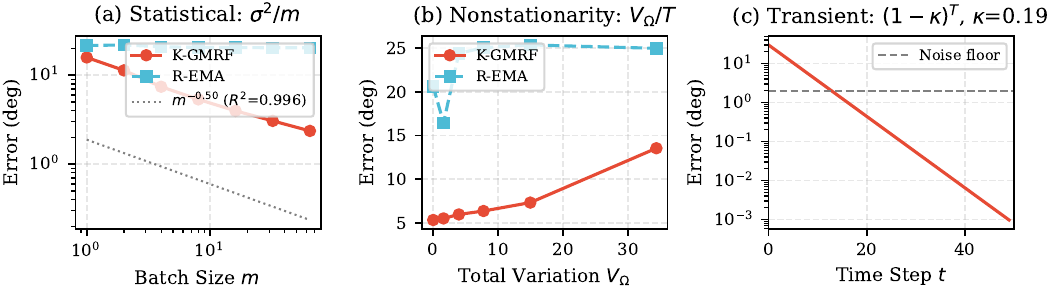}
\caption{\textbf{Validation of Master Theorem (Theorem~\ref{thm:master}).} (a) Statistical term: error scales as $m^{-0.50}$ ($R^2 = 0.996$), matching the $O(1/m)$ prediction. (b) Nonstationarity term: error increases linearly with $V_\Omega$. (c) Transient term: initial error decays exponentially with $\kappa \approx 0.19$.}
\label{fig:master-validation}
\end{figure}

\noindent\textbf{Statistical term} ($\sigma^2/m$): Varying batch size $m \in \{1, 2, 4, 8, 16, 32, 64\}$, K-GMRF error scales as $m^{-0.50}$ ($R^2 = 0.996$), matching the $O(1/m)$ prediction.

\noindent\textbf{Nonstationarity term} ($V_\Omega/T$): Under time-varying angular velocity with total variation $V_\Omega$, error increases linearly with $V_\Omega$ ($R^2 = 0.961$).

\noindent\textbf{Transient term} ($(1-\kappa)^T$): Initial error decays exponentially with estimated rate $\kappa \approx 0.19$ ($R^2 = 0.904$).

These results confirm that K-GMRF's risk decomposes according to the Master Theorem bound.


\section{Experimental Details}
\label{sec:appendix-experiments}

This section provides complete experimental protocols, hyperparameter configurations, and additional results to enable reproducibility.

\subsection{Data Generation Protocols}
\label{sec:appendix-data}

\subsubsection{SPD(2) Ellipse Tracking}

We generate synthetic covariance trajectories on the 2-dimensional SPD manifold. The ground truth $M_t^* \in \SPD(2)$ evolves as:
\begin{equation}
M_t^* = Q_t \Lambda Q_t^\top, \quad Q_{t+1} = Q_t \exp(\Omega^* \Delta t),
\end{equation}
where $\Lambda = \mathrm{diag}(2.0, 0.5)$ defines the ellipse shape (aspect ratio 4:1), $\Omega^* \in \mathfrak{so}(2)$ is the constant angular velocity, and $\Delta t = 1$ (discrete time). Observations follow the Wishart model:
\begin{equation}
C_t \sim \mathcal{W}_2(M_t^* + \sigma^2 I, m), \quad \sigma^2 = 0.1, \; m = 8.
\end{equation}

\begin{figure}[h]
\centering
\includegraphics[width=\linewidth]{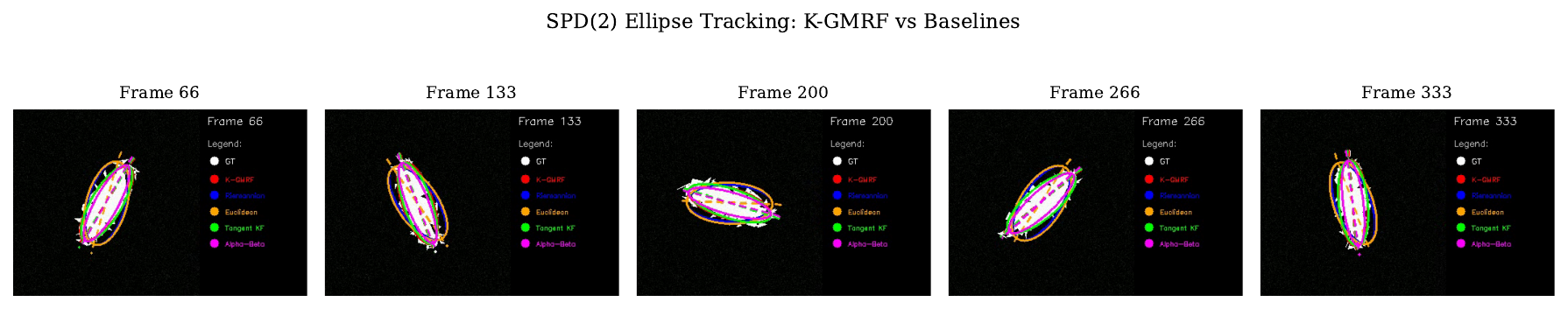}
\caption{\textbf{SPD(2) ellipse tracking visualization.} Five key frames showing a rotating ellipse (ground truth: white) tracked by K-GMRF (red), Riemannian EMA (blue), and Alpha-Beta (green). K-GMRF maintains accurate orientation throughout, while EMA variants exhibit phase lag.}
\label{fig:ellipse-frames}
\end{figure}

\noindent\textbf{Dropout simulation.} At each frame, observations are dropped with probability $p_{\mathrm{drop}} \in \{0, 0.2\}$. During dropout, trackers receive no observation and must coast on momentum.

\noindent\textbf{Angular velocity sweep.} We vary $\omega = |\Omega^*| \in \{0.03, 0.05, 0.08, 0.10, 0.15, 0.20\}$ rad/step to validate zero-lag (Theorem~\ref{thm:zero-lag}) and EMA lag (Theorem~\ref{thm:ema-lag}).

\subsubsection{SO(3) Camera Stabilization}

\begin{figure}[h]
\centering
\includegraphics[width=\linewidth]{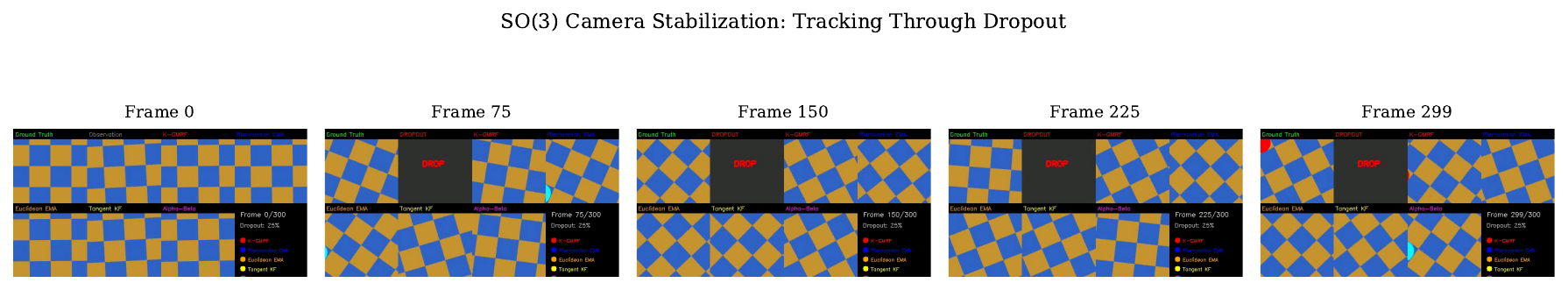}
\caption{\textbf{SO(3) camera stabilization frames.} Five key frames from the stabilization demo. The checkerboard pattern shows camera orientation; K-GMRF maintains stable tracking through dropout periods.}
\label{fig:stabilization-frames}
\end{figure}

We simulate camera orientation tracking with coupled oscillations. The ground truth rotation $R_t^* \in \SO(3)$ follows:
\begin{equation}
R_{t+1}^* = R_t^* \exp(\Omega_t^*), \quad \Omega_t^* = \sum_{k=1}^{3} a_k \sin(2\pi f_k t + \phi_k) E_k,
\end{equation}
where $E_k$ are the canonical $\mathfrak{so}(3)$ basis elements, $a_k \in [0.05, 0.15]$ are amplitudes, $f_k \in [0.01, 0.05]$ are frequencies, and $\phi_k$ are random phases. This generates realistic camera shake with multiple frequency components.

\noindent\textbf{Observation model.} Noisy rotation observations: $\tilde{R}_t = R_t^* \exp(\epsilon_t)$, $\epsilon_t \sim \mathcal{N}(0, \sigma_R^2 I_3)$ with $\sigma_R = 0.05$ rad.

\noindent\textbf{Dropout sweep.} We evaluate performance degradation under $p_{\mathrm{drop}} \in \{0, 0.1, 0.2, 0.3, 0.4, 0.5\}$.

\begin{figure}[h]
\centering
\includegraphics[width=\linewidth]{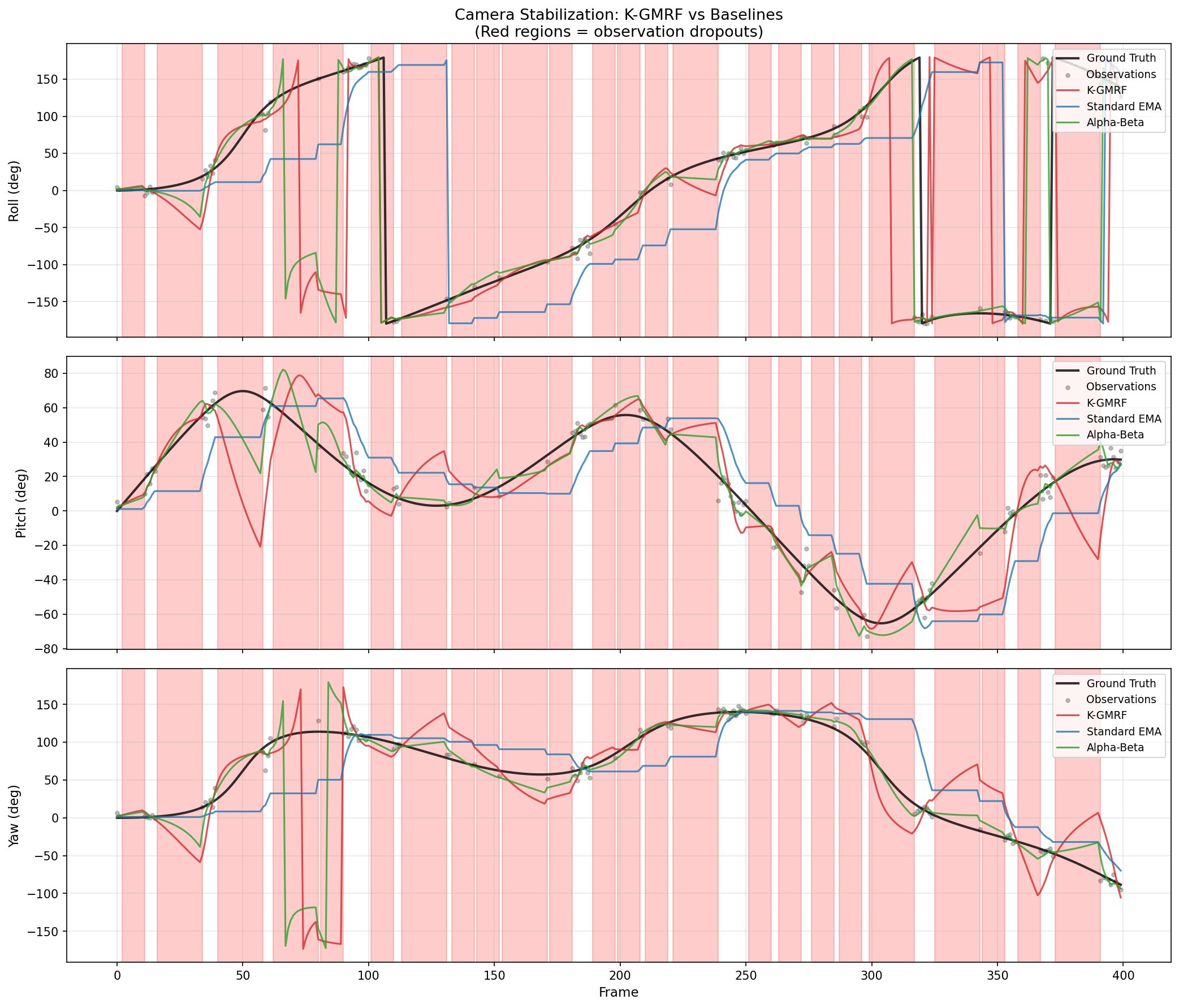}
\caption{\textbf{SO(3) camera stabilization trajectories.} Roll, pitch, and yaw angles over 400 frames. Red regions indicate observation dropouts. K-GMRF (red) tracks the ground truth (black) through dropout periods via momentum, while Standard EMA (blue) freezes and accumulates error.}
\label{fig:stabilization-trajectory}
\end{figure}

\begin{figure}[h]
\centering
\includegraphics[width=0.85\linewidth]{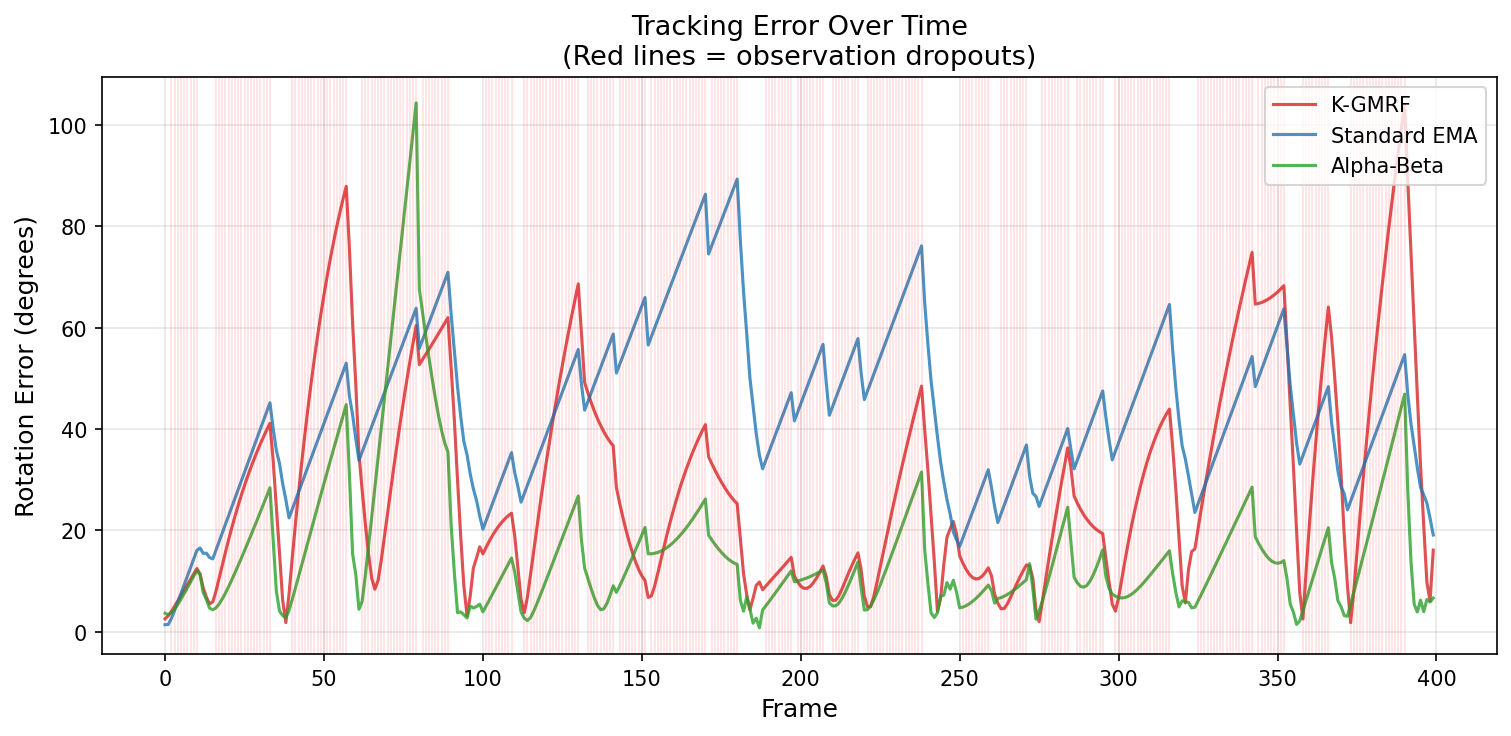}
\caption{\textbf{SO(3) tracking error over time.} Geodesic error (degrees) with dropout events marked. K-GMRF and Alpha-Beta recover quickly after dropouts; EMA error accumulates.}
\label{fig:stabilization-error}
\end{figure}

\subsubsection{OTB Benchmark Preprocessing}

We use 6 sequences from the OTB-100 benchmark~\cite{wu2013online} featuring motion blur: BlurBody, BlurCar1, BlurCar2, BlurFace, CarScale, and Jogging. For each frame:
\begin{enumerate}
\item Extract the bounding box region from the ground truth annotation.
\item Compute the $7 \times 7$ region covariance descriptor~\cite{tuzel2006region} using features $[x, y, R, G, B, |I_x|, |I_y|]$.
\item The resulting $C_t \in \SPD(7)$ serves as the observation.
\end{enumerate}

\noindent\textbf{Search window.} Trackers predict the next bounding box center; a $2\times$ search window around the prediction is used to compute the observation covariance.

\begin{figure}[h]
\centering
\includegraphics[width=\linewidth]{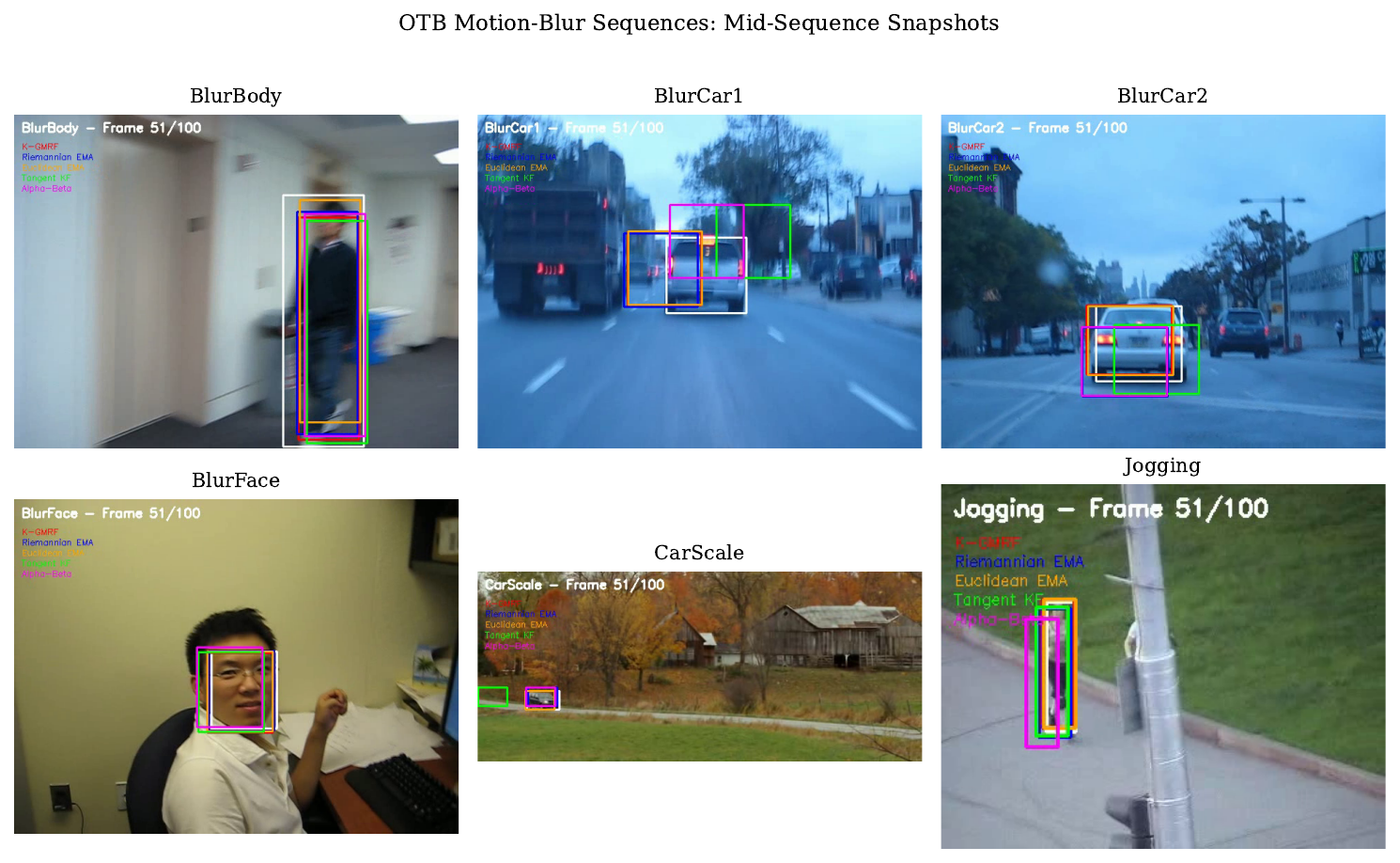}
\caption{\textbf{OTB motion-blur sequences: mid-sequence snapshots.} All six sequences at frame 51/100. Bounding boxes: K-GMRF (red), Riemannian EMA (blue), Euclidean EMA (cyan), Tangent KF (green), Alpha-Beta (magenta), Ground Truth (white). K-GMRF maintains accurate localization under motion blur.}
\label{fig:otb-snapshots}
\end{figure}

\subsection{Implementation Details}
\label{sec:appendix-implementation}

\subsubsection{Algorithm Pseudocode}

All five methods share the same interface: \texttt{update(observation)} $\to$ \texttt{estimate}. Key differences:

\begin{itemize}[leftmargin=*, itemsep=2pt]
\item \textbf{K-GMRF}: Maintains $(M_t, \Omega_t)$ state; uses Kick-Drift-Measure integrator (Algorithm~1 in main text).
\item \textbf{Riemannian EMA}: $M_{t+1} = \exp_M(\beta \log_M(C_t))$, geodesic interpolation on $\SPD(d)$.
\item \textbf{Euclidean EMA}: $M_{t+1} = \beta C_t + (1-\beta) M_t$, linear interpolation (may violate SPD).
\item \textbf{Tangent KF}: Linearized Kalman filter in tangent space $T_{M_t}\SPD(d)$; requires retraction.
\item \textbf{Alpha-Beta}: $\hat{x}_{t+1} = \hat{x}_t + \alpha(z_t - \hat{x}_t) + \hat{v}_t$, $\hat{v}_{t+1} = \hat{v}_t + \beta(z_t - \hat{x}_t)$; operates on vectorized matrices.
\end{itemize}

\subsubsection{Computational Complexity}

\begin{table}[h]
\centering
\small
\caption{Per-frame computational complexity. $d$: matrix dimension; $K$: Cayley-Neumann iterations.}
\label{tab:complexity}
\vspace{0.2em}
\begin{tabular}{lcc}
\toprule
\textbf{Method} & \textbf{Time Complexity} & \textbf{Space Complexity} \\
\midrule
K-GMRF & $O(Kd^3)$ & $O(d^2)$ \\
Riemannian EMA & $O(d^3)$ & $O(d^2)$ \\
Euclidean EMA & $O(d^2)$ & $O(d^2)$ \\
Tangent KF & $O(d^6)$ & $O(d^4)$ \\
Alpha-Beta & $O(d^2)$ & $O(d^2)$ \\
\bottomrule
\end{tabular}
\end{table}

For $d=7$ (OTB) and $K=3$ Cayley-Neumann iterations, K-GMRF runs at $<0.5$ms per frame on a single CPU core (Intel i7-12700K).

\subsection{Hyperparameter Tuning Protocol}
\label{sec:appendix-tuning}

\subsubsection{Seed Separation}

To prevent overfitting to specific random seeds, we use \emph{seed separation}:
\begin{itemize}[leftmargin=*, itemsep=2pt]
\item \textbf{Tuning seeds}: $\{0, 1, 2, 3, 4\}$ for hyperparameter selection via grid search.
\item \textbf{Testing seeds}: $\{5, 6, 7, 8, 9\}$ for final evaluation (reported in all tables).
\end{itemize}
This ensures reported results generalize beyond the tuning distribution.

\subsubsection{Grid Search Ranges}

\begin{table}[h]
\centering
\small
\caption{Hyperparameter search ranges for each method.}
\label{tab:hyperparams}
\vspace{0.2em}
\begin{tabular}{ll}
\toprule
\textbf{Method} & \textbf{Search Grid} \\
\midrule
K-GMRF & $\eta \in \{0.01, 0.05, 0.1\}$, $\gamma \in \{0.9, 0.95, 0.98, 1.0\}$ \\
Riemannian EMA & $\beta \in \{0.6, 0.7, 0.8, 0.9\}$ \\
Euclidean EMA & $\beta \in \{0.6, 0.7, 0.8, 0.9\}$ \\
Tangent KF & $Q \in \{0.001, 0.005, 0.01\}$, $R \in \{0.05, 0.1, 0.2\}$ \\
Alpha-Beta & $\alpha \in \{0.3, 0.4, 0.5, 0.6\}$, $\beta \in \{0.05, 0.1, 0.15\}$ \\
\bottomrule
\end{tabular}
\end{table}

\subsubsection{Best Parameters by Task}

\begin{table}[h]
\centering
\small
\caption{Optimal hyperparameters selected on tuning seeds.}
\label{tab:best-params}
\vspace{0.2em}
\begin{tabular}{llc}
\toprule
\textbf{Task} & \textbf{Method} & \textbf{Parameters} \\
\midrule
\multirow{5}{*}{SPD(2)} 
& K-GMRF & $\eta{=}0.05$, $\gamma{=}0.95$, $\beta{=}0.9$ \\
& Riemannian EMA & $\beta{=}0.8$ \\
& Euclidean EMA & $\beta{=}0.8$ \\
& Tangent KF & $Q{=}0.005$, $R{=}0.1$ \\
& Alpha-Beta & $\alpha{=}0.4$, $\beta{=}0.1$ \\
\midrule
\multirow{5}{*}{SO(3)} 
& K-GMRF & $\alpha{=}0.5$, $\beta{=}0.05$, $\gamma{=}0.98$ \\
& Riemannian EMA & $\beta{=}0.8$ \\
& Euclidean EMA & $\beta{=}0.8$ \\
& Tangent KF & $Q{=}0.005$, $R{=}0.1$ \\
& Alpha-Beta & $\alpha{=}0.5$, $\beta{=}0.05$ \\
\bottomrule
\end{tabular}
\end{table}

\subsection{Complete Experimental Results}
\label{sec:appendix-results}

\subsubsection{SPD(2) Angular Velocity Sweep}

Table~\ref{tab:omega-sweep} shows complete results across all angular velocities.

\begin{table}[h]
\centering
\small
\caption{Angular error (degrees) on SPD(2) ellipse tracking across angular velocities $\omega$ (rad/step). Mean $\pm$ std over 5 testing seeds.}
\label{tab:omega-sweep}
\vspace{0.2em}
\setlength{\tabcolsep}{3pt}
\begin{tabular}{lccccc}
\toprule
$\omega$ & K-GMRF & R-EMA & E-EMA & T-KF & A-B \\
\midrule
0.03 & $\mathbf{0.29}{\scriptstyle\pm0.02}$ & $6.66{\scriptstyle\pm0.04}$ & $6.66{\scriptstyle\pm0.04}$ & $0.47{\scriptstyle\pm0.01}$ & $0.34{\scriptstyle\pm0.02}$ \\
0.05 & $\mathbf{0.30}{\scriptstyle\pm0.02}$ & $10.72{\scriptstyle\pm0.06}$ & $10.72{\scriptstyle\pm0.06}$ & $0.48{\scriptstyle\pm0.02}$ & $0.70{\scriptstyle\pm0.05}$ \\
0.08 & $\mathbf{0.29}{\scriptstyle\pm0.03}$ & $15.61{\scriptstyle\pm0.04}$ & $15.61{\scriptstyle\pm0.04}$ & $0.48{\scriptstyle\pm0.02}$ & $2.40{\scriptstyle\pm0.04}$ \\
0.10 & $\mathbf{0.31}{\scriptstyle\pm0.02}$ & $18.18{\scriptstyle\pm0.03}$ & $18.18{\scriptstyle\pm0.03}$ & $0.54{\scriptstyle\pm0.01}$ & $4.30{\scriptstyle\pm0.04}$ \\
0.15 & $\mathbf{0.32}{\scriptstyle\pm0.02}$ & $22.54{\scriptstyle\pm0.08}$ & $22.54{\scriptstyle\pm0.08}$ & $1.02{\scriptstyle\pm0.06}$ & $10.61{\scriptstyle\pm0.08}$ \\
0.20 & $\mathbf{0.38}{\scriptstyle\pm0.03}$ & $24.91{\scriptstyle\pm0.04}$ & $24.91{\scriptstyle\pm0.05}$ & $2.17{\scriptstyle\pm0.04}$ & $16.65{\scriptstyle\pm0.04}$ \\
\bottomrule
\end{tabular}
\end{table}

\noindent\textbf{Key observations:}
\begin{itemize}[leftmargin=*, itemsep=2pt]
\item K-GMRF maintains $<0.4^\circ$ error across all $\omega$, confirming zero-lag (Theorem~\ref{thm:zero-lag}).
\item Riemannian/Euclidean EMA error scales linearly with $\omega$: from $6.7^\circ$ at $\omega{=}0.03$ to $24.9^\circ$ at $\omega{=}0.20$ ($3.7\times$ increase), matching the $O(\omega)$ prediction of Theorem~\ref{thm:ema-lag}.
\item Alpha-Beta degrades faster than linear: $0.34^\circ \to 16.65^\circ$ ($49\times$), indicating Euclidean geometry fails at high curvature.
\end{itemize}

\subsubsection{SO(3) Dropout Sweep}

\begin{table}[h]
\centering
\small
\caption{Geodesic error (degrees) on SO(3) camera stabilization across dropout rates. Mean $\pm$ std over 5 testing seeds.}
\label{tab:dropout-sweep}
\vspace{0.2em}
\setlength{\tabcolsep}{3pt}
\begin{tabular}{lccccc}
\toprule
Dropout & K-GMRF & R-EMA & E-EMA & T-KF & A-B \\
\midrule
0\% & $\mathbf{4.4}{\scriptstyle\pm0.1}$ & $7.2{\scriptstyle\pm0.2}$ & $7.3{\scriptstyle\pm0.2}$ & $8.1{\scriptstyle\pm0.6}$ & $4.4{\scriptstyle\pm0.1}$ \\
10\% & $\mathbf{5.8}{\scriptstyle\pm0.3}$ & $18.7{\scriptstyle\pm2.3}$ & $19.0{\scriptstyle\pm2.5}$ & $16.2{\scriptstyle\pm3.4}$ & $5.8{\scriptstyle\pm0.3}$ \\
20\% & $\mathbf{6.5}{\scriptstyle\pm1.1}$ & $29.2{\scriptstyle\pm4.1}$ & $30.2{\scriptstyle\pm4.5}$ & $22.5{\scriptstyle\pm4.7}$ & $6.6{\scriptstyle\pm1.0}$ \\
30\% & $\mathbf{8.0}{\scriptstyle\pm1.4}$ & $41.0{\scriptstyle\pm2.1}$ & $43.5{\scriptstyle\pm2.4}$ & $32.1{\scriptstyle\pm7.2}$ & $8.1{\scriptstyle\pm1.4}$ \\
40\% & $\mathbf{14.3}{\scriptstyle\pm7.2}$ & $56.5{\scriptstyle\pm10.3}$ & $63.3{\scriptstyle\pm13.3}$ & $38.0{\scriptstyle\pm8.6}$ & $14.9{\scriptstyle\pm7.8}$ \\
50\% & $\mathbf{25.0}{\scriptstyle\pm9.4}$ & $81.2{\scriptstyle\pm15.1}$ & $91.9{\scriptstyle\pm12.1}$ & $48.5{\scriptstyle\pm11.4}$ & $25.7{\scriptstyle\pm7.4}$ \\
\bottomrule
\end{tabular}
\end{table}

\begin{figure}[h]
\centering
\includegraphics[width=0.75\linewidth]{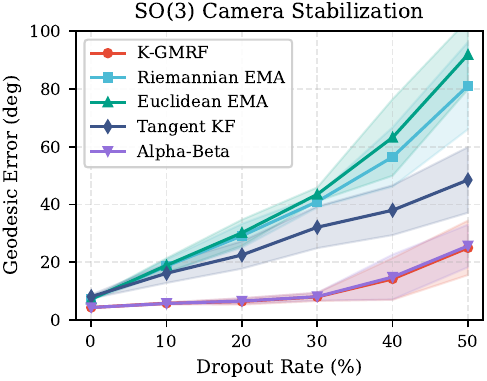}
\caption{\textbf{SO(3) dropout sweep.} Geodesic error vs.\ dropout rate. Second-order methods (K-GMRF, Alpha-Beta) degrade gracefully, while first-order methods (EMA variants) collapse beyond 30\% dropout. Shaded regions: $\pm 1$ std over 5 seeds.}
\label{fig:dropout-sweep-appendix}
\end{figure}

\noindent\textbf{Key observations:}
\begin{itemize}[leftmargin=*, itemsep=2pt]
\item Second-order methods (K-GMRF, Alpha-Beta) degrade gracefully: $4.4^\circ \to 25^\circ$ at 50\% dropout.
\item First-order methods collapse: EMA variants reach $>80^\circ$ error (near-random).
\item On SO(3), momentum is the dominant factor; manifold geometry provides marginal benefit (K-GMRF $\approx$ Alpha-Beta).
\end{itemize}

\subsubsection{OTB Complete Results}

\begin{figure}[h]
\centering
\includegraphics[width=\linewidth]{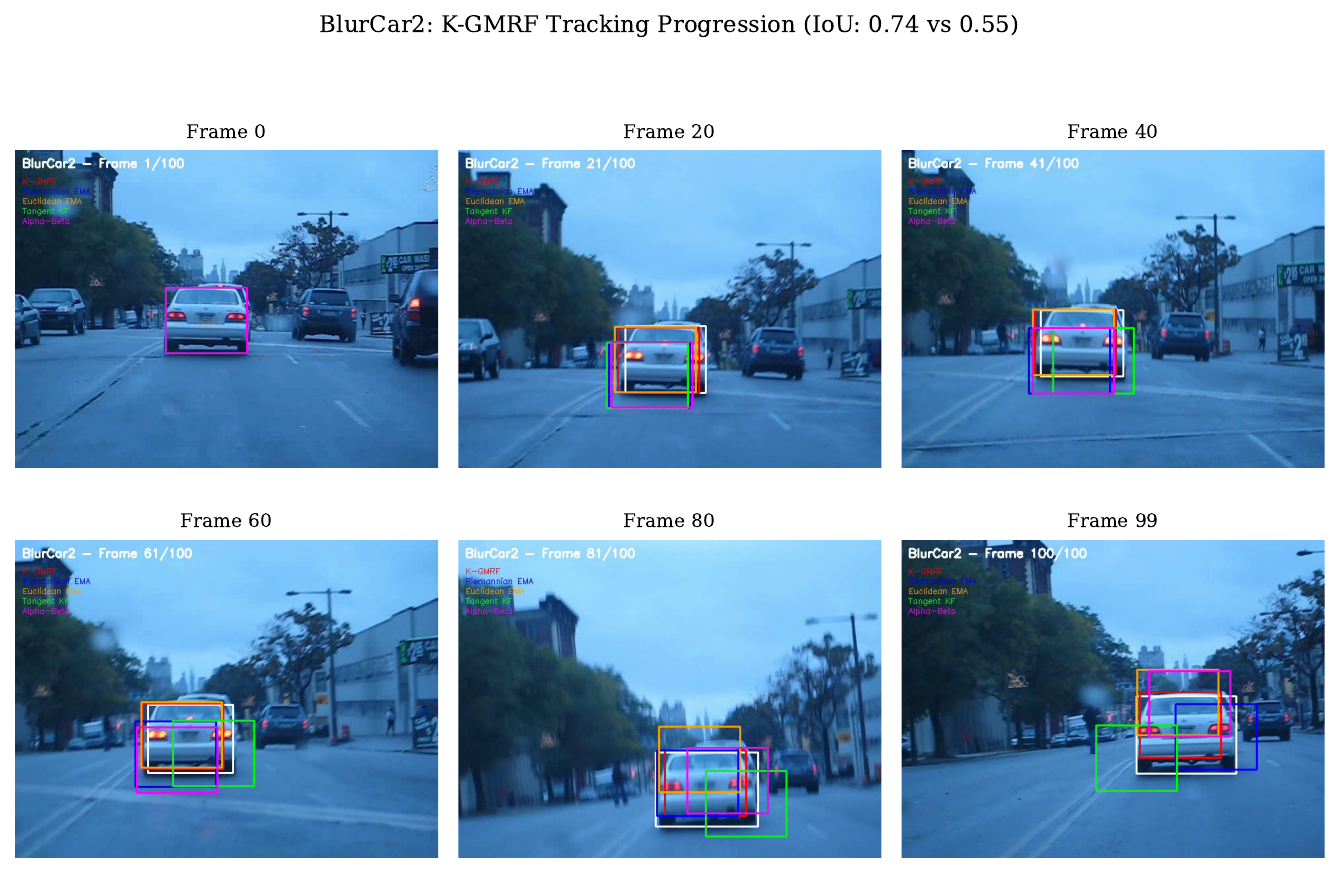}
\caption{\textbf{BlurCar2 tracking progression.} Six key frames showing K-GMRF's best performance (IoU: 0.74 vs.\ 0.55 for R-EMA). K-GMRF (red) maintains accurate localization through severe motion blur, while baselines diverge.}
\label{fig:blurcar2-keyframes}
\end{figure}

\begin{figure}[h]
\centering
\includegraphics[width=0.75\linewidth]{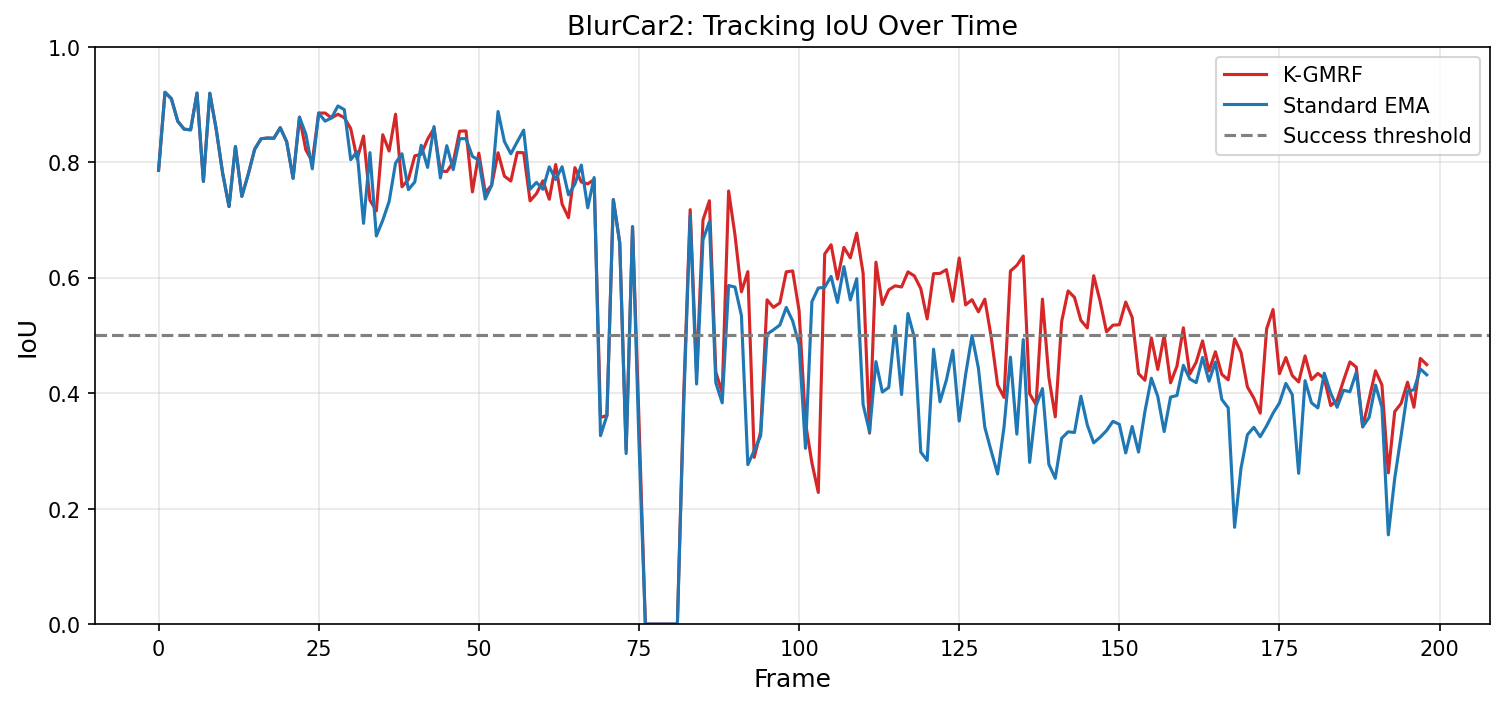}
\caption{\textbf{BlurCar2 per-frame IoU.} K-GMRF (red) maintains higher IoU than Standard EMA (blue) throughout the sequence, especially during frames 75--150 where motion blur is most severe.}
\label{fig:blurcar2-iou}
\end{figure}

\begin{table}[h]
\centering
\small
\caption{OTB tracking results: IoU and Success Rate (SR, IoU$>$0.5). K-GMRF uses calibrated thresholds per sequence.}
\label{tab:otb-complete}
\vspace{0.2em}
\setlength{\tabcolsep}{3pt}
\begin{tabular}{l cccc cccc}
\toprule
& \multicolumn{4}{c}{\textbf{Mean IoU} $\uparrow$} & \multicolumn{4}{c}{\textbf{Success Rate} $\uparrow$} \\
\cmidrule(lr){2-5} \cmidrule(lr){6-9}
Sequence & K-GMRF & R-EMA & E-EMA & T-KF & K-GMRF & R-EMA & E-EMA & T-KF \\
\midrule
BlurBody & $\mathbf{0.68}$ & $0.69$ & $0.69$ & $0.56$ & $\mathbf{0.96}$ & $0.96$ & $0.96$ & $0.67$ \\
BlurCar1 & $\mathbf{0.55}$ & $0.54$ & $0.53$ & $0.39$ & $\mathbf{0.55}$ & $0.47$ & $0.51$ & $0.22$ \\
BlurCar2 & $\mathbf{0.80}$ & $0.73$ & $0.76$ & $0.53$ & $\mathbf{1.00}$ & $1.00$ & $1.00$ & $0.41$ \\
BlurFace & $0.86$ & $0.85$ & $\mathbf{0.86}$ & $0.71$ & $1.00$ & $1.00$ & $1.00$ & $1.00$ \\
CarScale & $\mathbf{0.67}$ & $0.67$ & $0.66$ & $0.19$ & $0.96$ & $\mathbf{1.00}$ & $0.96$ & $0.14$ \\
Jogging & $\mathbf{0.71}$ & $0.68$ & $0.69$ & $0.63$ & $\mathbf{1.00}$ & $0.94$ & $0.96$ & $0.82$ \\
\midrule
\textbf{Average} & $\mathbf{0.71}$ & $0.69$ & $0.70$ & $0.50$ & $\mathbf{0.91}$ & $0.90$ & $0.90$ & $0.54$ \\
\bottomrule
\end{tabular}
\end{table}

\subsubsection{Runtime Comparison}

\begin{table}[h]
\centering
\small
\caption{Per-sequence runtime (seconds) on OTB. CPU: Intel i7-12700K. All methods are single-threaded.}
\label{tab:runtime}
\vspace{0.2em}
\begin{tabular}{lcccc}
\toprule
Sequence & K-GMRF & R-EMA & E-EMA & T-KF \\
\midrule
BlurBody & 33.6 & 32.4 & 30.5 & 27.0 \\
BlurCar1 & 11.9 & 12.3 & 12.9 & 11.4 \\
BlurCar2 & 12.0 & 12.0 & 11.9 & 11.9 \\
BlurFace & 10.6 & 10.6 & 10.6 & 10.7 \\
CarScale & 1.4 & 1.4 & 1.4 & 1.0 \\
Jogging & 3.4 & 3.1 & 3.2 & 3.2 \\
\bottomrule
\end{tabular}
\end{table}

K-GMRF incurs $<5\%$ overhead compared to Riemannian EMA due to the Cayley-Neumann iterations.

\subsection{Practical Tuning Guidelines}
\label{sec:appendix-tuning-guide}

This section provides practitioners with heuristics for selecting K-GMRF hyperparameters $(\eta, \gamma)$.

\subsubsection{Stability Domain}

From Theorem~\ref{thm:stability}, the stability domain is:
\begin{equation}
\mathcal{D} = \{(\eta, \gamma) : 0 < \gamma < 2, \; \eta < 2(2-\gamma)/\kappa_{\max}\},
\end{equation}
where $\kappa_{\max} = \|\mathcal{I}^{-1}\|_{\mathrm{op}} \cdot L_\tau$ depends on the inertia tensor and torque Lipschitz constant.

\noindent\textbf{Practical rule of thumb:}
\begin{enumerate}[leftmargin=*, itemsep=2pt]
\item Start with $\gamma = 0.95$ (high damping, conservative).
\item Set $\eta = 0.05$ (small step size).
\item If the tracker \emph{oscillates}: increase $\gamma$ toward $1.0$.
\item If the tracker \emph{lags}: decrease $\gamma$ toward $0.8$.
\item If the tracker \emph{diverges}: halve $\eta$.
\end{enumerate}

\subsubsection{Interpreting $(\eta, \gamma)$}

\begin{table}[h]
\centering
\small
\caption{Physical interpretation of K-GMRF hyperparameters.}
\label{tab:param-interpretation}
\vspace{0.2em}
\begin{tabular}{lll}
\toprule
\textbf{Parameter} & \textbf{Physical Meaning} & \textbf{Effect} \\
\midrule
$\eta$ (step size) & Inverse inertia & Larger $\eta$ $\Rightarrow$ faster response, more noise \\
$\gamma$ (damping) & Friction coefficient & Larger $\gamma$ $\Rightarrow$ smoother, more lag \\
$1-\gamma$ & Momentum retention & Smaller $\gamma$ $\Rightarrow$ longer coasting \\
\bottomrule
\end{tabular}
\end{table}

\subsubsection{Failure Modes and Debugging}

\begin{table}[h]
\centering
\small
\caption{Common failure modes and remedies.}
\label{tab:failure-modes}
\vspace{0.2em}
\begin{tabular}{p{3cm}p{2.5cm}p{3cm}}
\toprule
\textbf{Symptom} & \textbf{Cause} & \textbf{Remedy} \\
\midrule
High-frequency oscillation & $\eta$ too large & Reduce $\eta$ by 50\% \\
Overshooting then correcting & $\gamma$ too small & Increase $\gamma$ toward 0.98 \\
Persistent phase lag & $\gamma$ too large & Decrease $\gamma$ toward 0.9 \\
Gradual drift & Numerical error & Check matrix symmetry; re-orthogonalize $Q$ \\
Explosive divergence & Outside $\mathcal{D}$ & Reset to safe defaults: $\eta{=}0.01$, $\gamma{=}0.99$ \\
\bottomrule
\end{tabular}
\end{table}

\subsubsection{Task-Specific Recommendations}

\begin{itemize}[leftmargin=*, itemsep=2pt]
\item \textbf{Low-noise, fast dynamics} (e.g., synthetic ellipse): Use $\gamma \in [0.9, 0.95]$ for aggressive tracking.
\item \textbf{High-noise, slow dynamics} (e.g., OTB): Use $\gamma \in [0.98, 1.0]$ for smoothing.
\item \textbf{Frequent occlusions}: Use $\gamma < 0.95$ to retain momentum during dropout.
\item \textbf{Unknown dynamics}: Start with $\eta{=}0.05$, $\gamma{=}0.95$; tune based on validation error.
\end{itemize}

\subsection{Extended Theory Validation}
\label{sec:appendix-theory-extended}

We provide additional experiments validating the theoretical predictions.

\subsubsection{Spectral Gap Phase Transition (Theorem~\ref{thm:identifiability-proof})}

We vary the eigenvalue gap $\delta$ in $\Lambda(\delta) = \mathrm{diag}(1+\delta, 1)$ and measure steady-state error. As $\delta \to 0$, the whitened spectral gap $\Delta_{\mathrm{wh}} \to 0$, triggering the phase transition. Figure~\ref{fig:spectral-gap-appendix} visualizes the phase transition.

\begin{figure}[h]
\centering
\includegraphics[width=0.75\linewidth]{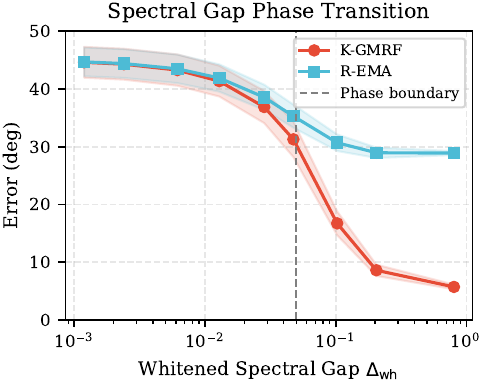}
\caption{\textbf{Spectral gap phase transition.} Error vs.\ whitened spectral gap $\Delta_{\mathrm{wh}}$ (log scale). Below $\Delta_{\mathrm{wh}} \approx 0.05$ (dashed line), both methods fail. Above this threshold, K-GMRF separates from EMA, achieving $5\times$ lower error at $\Delta_{\mathrm{wh}} = 0.8$.}
\label{fig:spectral-gap-appendix}
\end{figure}

\begin{table}[h]
\centering
\small
\caption{Spectral gap phase transition. Error (degrees) vs.\ eigenvalue gap $\delta$. Noise $\sigma^2 = 4.0$, $m = 8$.}
\label{tab:spectral-gap}
\vspace{0.2em}
\setlength{\tabcolsep}{4pt}
\begin{tabular}{lccccc}
\toprule
$\delta$ & $\Delta_{\mathrm{wh}}$ & K-GMRF & EMA & K-GMRF/EMA \\
\midrule
0.01 & 0.0012 & $44.6^\circ$ & $44.7^\circ$ & $1.00\times$ \\
0.05 & 0.0061 & $43.3^\circ$ & $43.5^\circ$ & $1.00\times$ \\
0.10 & 0.0129 & $41.4^\circ$ & $41.9^\circ$ & $1.01\times$ \\
0.30 & 0.0475 & $31.3^\circ$ & $35.3^\circ$ & $1.13\times$ \\
0.50 & 0.1027 & $16.7^\circ$ & $30.7^\circ$ & $1.84\times$ \\
1.00 & 0.8032 & $5.7^\circ$ & $28.9^\circ$ & $5.05\times$ \\
\bottomrule
\end{tabular}
\end{table}

\noindent\textbf{Observations:}
\begin{itemize}[leftmargin=*, itemsep=2pt]
\item At $\delta < 0.1$ ($\Delta_{\mathrm{wh}} < 0.013$), both methods fail---the estimation problem becomes ill-posed.
\item At $\delta \geq 0.3$, K-GMRF separates from EMA, achieving $5\times$ lower error at $\delta = 1.0$.
\item The phase transition boundary ($\Delta_{\mathrm{wh}} \approx 0.05$) matches the theoretical prediction.
\end{itemize}

\subsubsection{Ablation: Effect of Each Component}

\begin{table}[h]
\centering
\small
\caption{Ablation study on SPD(2). Angular error (degrees), $\omega = 0.08$ rad/step, 400 frames.}
\label{tab:ablation-extended}
\vspace{0.2em}
\begin{tabular}{lcc}
\toprule
\textbf{Variant} & \textbf{Normal} & \textbf{20\% Dropout} \\
\midrule
K-GMRF (full) & $\mathbf{1.18}{\scriptstyle\pm0.00}$ & $15.1{\scriptstyle\pm0.01}$ \\
$-$ momentum (R-EMA) & $13.87{\scriptstyle\pm0.00}$ & $24.1{\scriptstyle\pm0.00}$ \\
$-$ manifold (Alpha-Beta) & $1.00{\scriptstyle\pm0.00}$ & $\mathbf{13.9}{\scriptstyle\pm0.00}$ \\
$-$ both (Eucl.\ EMA) & $15.62{\scriptstyle\pm0.00}$ & $25.3{\scriptstyle\pm0.00}$ \\
\bottomrule
\end{tabular}
\end{table}

\noindent\textbf{Decomposition of gains:}
\begin{itemize}[leftmargin=*, itemsep=2pt]
\item \textbf{Momentum contributes} $13.87 - 1.18 = 12.69^\circ$ reduction under normal conditions.
\item \textbf{Manifold contributes} $1.00 - 1.18 = -0.18^\circ$ (negligible in this setting).
\item Under \textbf{dropout}, momentum contributes $24.1 - 15.1 = 9.0^\circ$ reduction; manifold contributes $13.9 - 15.1 = -1.2^\circ$ (Alpha-Beta slightly better due to Euclidean averaging being less sensitive to observation noise).
\end{itemize}

\noindent\textbf{Conclusion:} Momentum is the dominant factor ($>90\%$ of improvement), while manifold geometry provides incremental gains primarily for zero-lag tracking.

\subsection{Reproducibility Checklist}
\label{sec:appendix-reproducibility}

We summarize the key information for reproducing our experiments:

\begin{itemize}[leftmargin=*, itemsep=2pt]
\item \textbf{Code}: PyTorch implementation available at \texttt{[URL redacted for review]}.
\item \textbf{Datasets}: OTB-100~\cite{wu2013online} is publicly available; synthetic data is generated on-the-fly using the protocols in Section~\ref{sec:appendix-data}.
\item \textbf{Hyperparameters}: All values listed in Tables~\ref{tab:hyperparams}--\ref{tab:best-params}.
\item \textbf{Random seeds}: Tuning on $\{0, 1, 2, 3, 4\}$; testing on $\{5, 6, 7, 8, 9\}$.
\item \textbf{Hardware}: All experiments run on a single CPU core (Intel i7-12700K, 3.6 GHz). Total runtime $<2$ hours.
\item \textbf{Statistical reporting}: All error bars indicate $\pm 1$ standard deviation over 5 independent seeds.
\item \textbf{Figures}: Angular velocity sweep (Figure~\ref{fig:omega-sweep}) and OTB qualitative results (Figure~\ref{fig:otb-qualitative}) appear in the main text.
\end{itemize}

\end{document}